%% file: SIMODS_main.tex
\documentclass[review,onefignum,onetabnum]{siamonline190516}
\pdfoutput=1
\input{SIMODS_shared}

\usepackage{times}
\usepackage{epsfig}
\usepackage{graphicx}
\usepackage{amssymb}
\usepackage[toc,page]{appendix}
\usepackage{microtype}
\usepackage{graphicx}
\usepackage{subfigure}
\usepackage{caption}
\usepackage{booktabs} % for professional tables
\usepackage[super]{nth}
\usepackage[font=scriptsize]{caption}
\usepackage{mwe}
\usepackage{mathtools}

\DeclarePairedDelimiter\floor{\lfloor}{\rfloor}

\usepackage{amsmath,amssymb,amsfonts,color}
\usepackage{graphicx,subfigure}
\usepackage{multirow}
\usepackage{booktabs}
\usepackage{hyperref}
\usepackage{amsmath}

\usepackage[utf8]{inputenc}
\usepackage[english]{babel}
\usepackage{xr}
\nolinenumbers

\renewcommand\labelenumi{(\roman{enumi})}
\renewcommand\theenumi\labelenumi

\newcommand{\comment}[1]{}

\begin{document}
\maketitle
%%%%%%%%% TITLE

%%%%%%%%% ABSTRACT
\begin{abstract}
We propose a probe for the analysis of deep learning architectures that is based on machine learning and approximation theoretical principles. Given a deep learning architecture and a training set, during or after training, the Sparsity Probe allows to analyze the performance of intermediate layers by quantifying the geometrical features of representations of the training set. We show how the Sparsity Probe enables measuring the contribution of adding depth to a given architecture, to detect under-performing layers, etc., all this without any auxiliary test data set.
\end{abstract}

% REQUIRED

% REQUIRED

%%%%%%%%% BODY TEXT
\section{Introduction}\label{sec:introduction}
Deep Neural Networks(DNN) have triumphantly improved benchmarks in a variety of different tasks. Remarkable works of architecture design \cite{lecun2015deeplearning, Resnets, Transformers}, optimization methods \cite{10.1214/aoms/1177729586, 2014arXiv1412.6980K, 2021arXiv210111075D}, and data mutation \cite{2015arXiv150203167I, JMLR:v15:srivastava14a, 2018arXiv180509501C, cheung2021modals} have been introduced and shown to empirically advance the fields of computer vision and natural language processing. Still, practitioners fail to justify the success of these models and lack the tools to test their real-world performance. Furthermore, while the basic notions of network architecture are understood \cite{lecun2015deeplearning, 10.5555/3086952}, it is difficult to assess the contribution of a certain layer of a trained model. Given these limitations, it is often unknown how to analyze a given architecture and how it can be improved. In many cases networks are treated as black boxes that lack explainability, and architectural experiments are conducted in a trial and error manner. An auxiliary test set is regularly presented as an approximation of the true dataset distribution \cite{10.5555/546466, 10.5555/993597, 10.5555/1643031.1643047} and used to quantify the model performance. 

Given the supervised classification setting, as presented in \cite{Ch}, any machine learning algorithm seeks to find a geometrical transformation that separates the samples of different categories and gathers the samples from matching categories. This notion has been prevalent in the field of Self-Supervised Learning(SSL) \cite{2020arXiv200607733G, simclr, 2021arXiv210303230Z}. In this paradigm, lacking the categorical information, a distorted image is generally compared to itself, to enforce the geometrical notion. 

In the Deep Learning setting, the category labels are often represented as a one-hot-encoding \cite{10.5555/1201987}, a vector in $\mathbb{R}^L$, where $L$ is the number of categories. 
Intermediate Features have been shown to learn incrementally higher-level features throughout the model layers \cite{erhan2009visualizing, olah2018the}. The output of the model's $k^\textrm{th}$ layer, as k grows, is expected to have a simpler structure, as the features contain more information that is class-specific. This concept corresponds to simpler mappings from incremental layers to the output labels. 

In the approximation theory approach, the sparsity of a function given some representation can be a robust method for evaluating its simplicity \cite{Elad, ED}. Functions in Deep Learning are generally not of a Sobolev nature, but rather in a general Besov Space \cite{ConstructiveApprox}. The study of adaptive, nonlinear approximation\cite{devore_1998}, allows the computation of this complexity score on such inherently non-smooth functions.

\begin{figure}[t]
\begin{center}
   \includegraphics[width=1.\linewidth]{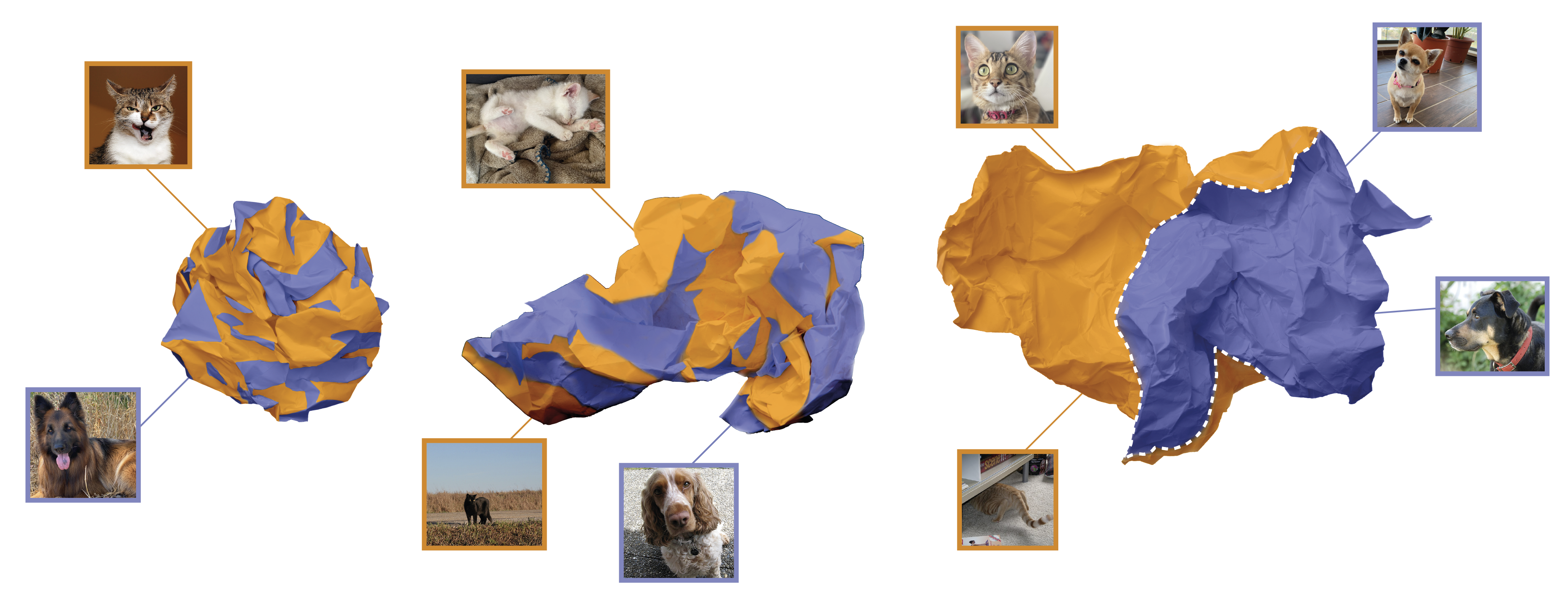}
\end{center}
   \caption{Demonstration of uncrumpling of the data representation through DL layers, as proposed in \cite{Ch}. 
    }
\label{folded_ball}
\label{fig:long}
\label{fig:onecol}
\end{figure}

\begin{figure}[t]
\begin{center}
   \includegraphics[width=1.\linewidth]{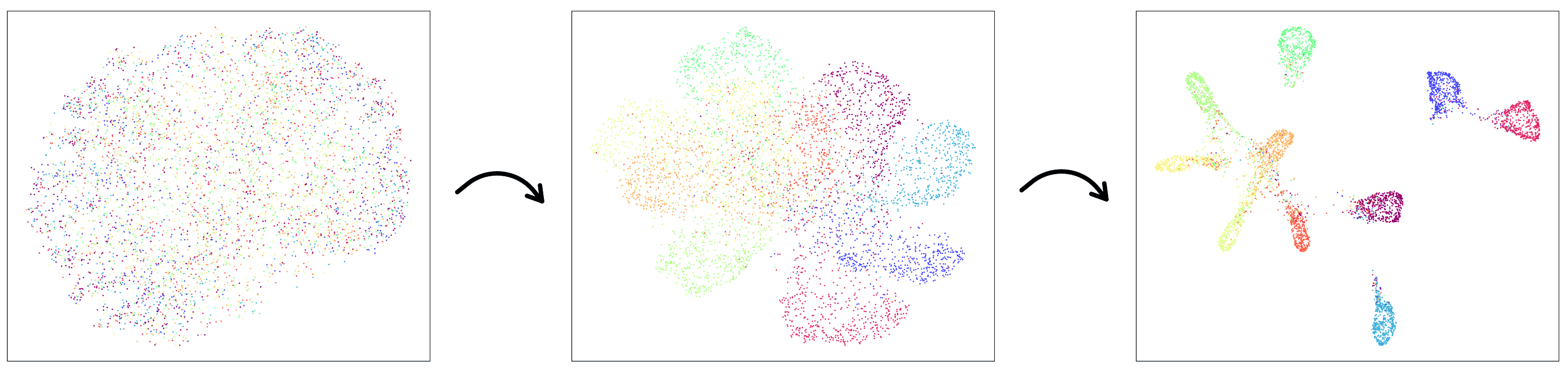}
\end{center}
   \caption{UMAP dimensionality reduction for the feature space of the input layer, the 2nd layer, and the 6th layer of a VGG-13 architecture well-trained on the CIFAR10 dataset. The improved clustering of the data representations is visually significant. 
    }
\label{mnist_umap}
\label{fig:long}
\label{fig:onecol}
\end{figure}

The main contributions of this paper are as follows:

%This section needs to be imrpoved..
\begin{enumerate}
%\begin{itemize}
  \item {\textbf{Sparsity-Probe}, a mathematically grounded tool for investigating the performance of intermediate model layers is introduced, using solely the train data, and model architecture(without an auxiliary test set).}
  \item {Extensive analysis is conducted showing the advantage of the Sparsity-Probe over  classical clustering indices.}
  \item {To support the theoretical basis, Sparsity-Probe is demonstrated on several known classification datasets.}
  \item {We present examples where the Sparsity-Probe is able to detect faulty or buggy architectures and by pinpointing  the problematic layer allows to fix them. }
%\end{itemize}
\end{enumerate}

%-------------------------------------------------------------------------
\section{Related Research and Concepts}\label{sec:Related_Research}

%-------------------------------------------------------------------------
\subsection{Statistical Approach}\label{sec:StatisticalApproach}
Different statistical and mathematical theories that aim to explain the success of DL have been proposed. The authors of \cite{IB} provide an Information-Bottleneck theory in which the network is viewed as a Markov-Chain. The Mutual-Information is documented between the inputs and the labels throughout the layers. Other statistical approaches \cite{Bahri} propose equivalence between increasingly-wide networks and Gaussian Processes. 

Classic Machine Learning algorithms rely heavily on the geometry of the input feature for their success. Methods like Support Vector Machines \cite{cortes1995support}, KNN \cite{MacQueen1967SomeMF}, and  Random Forest \cite{Br1} focus on leveraging the geometry of the data for training. DL models are required as automatic feature engineering tools, when there is no clear clustering of the classes in the original feature space (e.g. pixel representation in computer vision). This leads us to believe that without a clear understanding of the geometry in the hidden layers, one cannot hope to understand the prediction quality of the model.

\subsection{Approximation-Theoretical Approach}\label{sec:ApproximationApproach}
Approximation Theory has given great importance in the field of Signal and Image Processing. Many methods have been offered for uncovering concealed relevant information from signals \cite{ConstructiveApprox, Da, Elad, 10.5555/1525499}. There has been a large amount of interest in grounding the theoretical basis of Neural Networks from an approximation theoretical perspective.  Many such works study the expressive power of deep feed-forward neural networks(FNN) for a certain target function $f\in\mathcal{B}$ networks, where $\mathcal{B}$ is a given Banach Space. In \cite{Shenz1}, the number of neurons is used to characterize the approximation rate for H\"{o}lder continuous functions using ReLU FNNs. Using both the width and depth of the network, \cite{Shenz3} achieve optimal approximation characterization of deep ReLU networks for smooth functions. Upper and lower bounds for the capacity needed to approximate Sobolev Functions are demonstrated in \cite{YAROTSKY2017103}. The authors were able to show that deep ReLU networks are able to more efficiently approximate smooth functions than shallow networks. General continuous functions are considered in \cite{2018arXiv180203620Y}, where optimality is shown for constant-width fully connected in terms of approximation rates. 

However, typically the authors assume that the input dataset can be represented as samples of a continuous or smooth function, such as functions in certain Sobolev spaces with sufficiently high smoothness index. Yet, evidence suggests that in most computer vision problems the input space is more correctly modelled by a discontinuous function. UMAP \cite{UMAP} is a nonlinear dimensionality reduction method that is commonly used to visualize data. In \ref{mnist_umap} we visualize the feature space of the input layer, the 2nd layer, and the 6th layer of a VGG-13 architecture trained on the CIFAR10 dataset. We sample 6000 random instances from the Train Set, and fit the UMAP reduction on their matching latent representation. It is obvious that the input space represents a discontinuous function. It is eyeopening to see that throughout the layers, the geometric clustering is apparent. Thus, in this work, we assume that the input dataset can be modeled as a function in some geometric Besov space, with relatively low smoothness index.

\subsection{Sparsity-Based Approach}\label{sec:Sparsity Approach}
Sparsity has been shown paramount for representing complex signals and giving insights into their nature e.g. Wavelet and Fourier transforms \cite{10.5555/1525499, Da, Elad}. It is intuitive to believe that DNNs employ sparsity methods to achieve successful representation learning. The Multi-Layer Convolutional Sparse Coding(ML-CSC) \cite{2016arXiv160708194P} provides a sparsity-based apprehension of Convolutional Neural Networks. Given a Dictionary $D$, it is shown that a ReLU Network forward pass is in fact equivalent to a layer-wise Nonnegative Sparse Coding pursuit, using Soft Nonnegative layered thresholding as a sparsity pursuit approximation. In \cite{2017arXiv170808705S}, a holistic pursuit is proposed along with a method for such Dictionary Learning. The discovery that concatenated layers are sparse with respect to a proposed dictionary is an important one and helps bridge the gap between the sparsity theory and empirically found neural network architectures. Our study differs from this approach in two critical aspects. The ML-CSC model proves that under certain conditions, a sparse Dictionary and representation vectors can be found, and propose methods for finding them. In our work, a general Post-hoc technique is shown to reliably enhance the explainability of any given trained model, along with its train dataset. This, in turn, does not involve learning a specific dictionary and representation, but rather assessing the quality of a given state. More importantly, we focus on supervised learning, where the sample categories(whether provided or not) are integral to approximate the model quality. Indeed, our premise relies on the fact that the sparsity should be centered around the mapping between the latent features and the labels. To strengthen this claim, consider a certain representation space. For a specific label assignment, this representation can be extremely well clustered, yet completely intertwined for a different label assignment.

\subsection{Linear and Kernel Probes}\label{sec:Separability}
\begin{definition}[Linear Separability]
Two sets $\Omega_1,\Omega_2\subset \mathbb{R}^n$ are linearly separable, if their convex hulls do not intersect.
\end{definition}
\begin{definition}[Non-Linear Separability]\label{sep_by_curve}
We say the sets $\Omega_1,\dots \Omega_k \subset \mathbb{R}^n$ are non-linearly separable if for every $\Omega_i$ there exists a domain $M_{i}\subset\mathbb{R}^{n}$ with a smooth boundary, such that $\Omega_{i}\subset{{M}_{i}}$ and $M_{i}\cap\bigcup\limits_{j\ne{i}}^{}{M_j}=\emptyset$.
\end{definition}
Classic machine learning algorithms like SVMs and CART\cite{Loh2011ClassificationAR} seek to find the best separation in the feature space between clusters of different classes. In the field of Representation Learning  \cite{Ben}, contrastive losses \cite{simclr,  supervised_contrastive, MoCo} aim to separate samples of different underlying category, whilst clustering samples of the same category. It is then of interest to quantify the wellness of separability between classes in the latent space.

Recent works propose to compute the linear separability of the intermediate layers \cite{LP, Separability_and_geometry}. Linear Separability is simple to define and compute, yet fails to grasp any separation which is not linear.

We present three synthetic toy datasets, with feature space of dimension 2, and
two outcome classes: Spiral, Circles, and Gaussian Quantiles(GQ) - see figure \ref{synthetic_dataset_fig}. It is clear
that the Linear Classifier methods cannot differentiate between the classes, as they are not
linearly separable. 
A more sophisticated measure of separability is considered in \cite{Kernel_Analysis}, by using radial basis kernel PCA to map the latent space to a different representation (selecting $d$ leading singular values), and measure linear separability in the projected dimension. This is problematic as $d$ is difficult to choose. In fact, if $d$ is large enough, linear separability becomes trivial, and so the authors consider small $d$. In general, one should prefer to measure the true separability in the given dimension, instead of measuring the linear separability in a projected (potentially lower-dimensional) space.

The separability of feature space in an intermediate layer is equivalent to the smoothness of the mapping to the sample labels, in the one-hot-encoding scheme. This work relies on \textbf{sparsity} to measure the smoothness of such functions. In \cite{Ch}, a DNN is compared to an uncrumpling of a high dimensional paper ball, such that every layer decouples between the classes incrementally. This concept is visualized in \ref{folded_ball}. In this visualization, at the final layer, we show an example of data that are well separated, yet clearly not linearly separable. We propose the Sparsity-Probe as a tool to quantify this type of separation. 

\begin{figure}[t]
\begin{center}
   \includegraphics[width=0.8\linewidth]{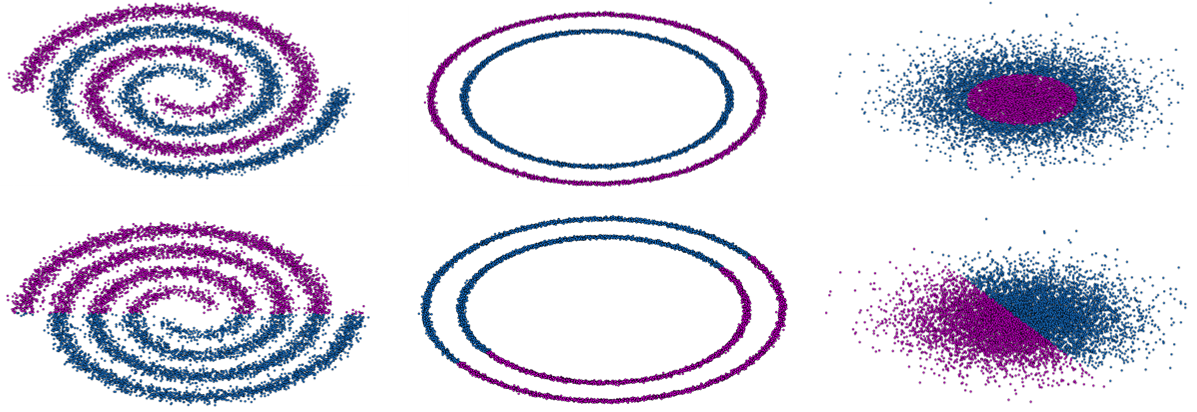}
\end{center}
   \caption{Three synthetic datasets which are all non-linearly separable: Top: Spiral, Circles, and Gaussian Quantiles datasets are shown. Bottom: Output of Linear Probes on these datasets. It is clear to see that a Linear Probe cannot capture the sparsity of such a function.}
\label{synthetic_dataset_fig}
\label{fig:long}
\label{fig:onecol}
\end{figure}

%-------------------------------------------------------------------------
\section{Formulation}\label{sec:formulation}
\subsection{Preliminaries}
Evaluating the sparsity of a high dimensional function can be a daunting task. We base our notion on a classic approach that arrives from Image Compression \cite{DJL}, where Wavelets\cite{Da} are used to approximate a signal. The Geometric Wavelet \cite{10.1137/040604649}, was shown to extend this notion to the terms of adaptive non-linear approximation. The importance of function sparsity in terms of of signal processing and representations has been thoroughly emphasized in \cite{Elad} . An in-depth introduction to Geometric Wavelets is shown in \cite{ED}.

In order to use the appropriate functional tools, we first need to state the problem in a functional setting. Assume we have a square image input sample for the model of side length $N$. We normalize its values and unravel the image into a long vector of size $N^2$. At each output layer, we again normalize the representation and unravel them into a vector, these will be the inputs for our functions. We now need to address the labels. In the multiclass classification setting, a common representation for a label is the one-hot-encoding. We use these encodings to represent every label as a vector in $\mathbb{R}^{L}$, for ${L}$ - the number of classes. At each layer, the unraveled vector representations along with their vector label value, are considered as samples of a vector-valued function associated with the layer.   

The complete NN can be modeled as a function $f:\mathbb{R}^{N^2}\rightarrow \mathbb{R}^L$. At the same time, for each layer $k$, assume there exists a function $f_k: \mathbb{R}^{n_k} \rightarrow \mathbb{R}^L $, that maps the unraveled feature vector of the $k$-th layer, to the label vector representation in $\mathbb{R}^{L}$. Certainly, any  sample of the training set produces simultaneously a sample for each of these functions $f_k$. A NN $f$ is obviously well-trained, if for a given input $x$ with label $y$, $f(x)$ is close to the one-hot-encoding of $y$. Furthermore, for $x_1, ..., x_n$ of the same underlying class, a well trained network will aim to cluster their representations at the intermediate levels. The most common loss functions are built to do just this. Although this is the penalty that is minimized, we claim a well trained model also aims to cluster the intermediate layers, based on the GT labels, and so yields a more clustered representation to be passed to the following layer. To state this more thoroughly,  We argue that in a well trained network, each input space for such $f_k$ is more clustered in terms of class label, and so the function mapping it to the class label is smoother. Let us refine the concept of such clustering. Although we do wish that input samples of the same label be close in the $k^{th}$ output space, we need a measure that captures the possibility of several different clusters from a same class. This is a slightly different criterion than that of clustering. We are looking for a notion of good behavior, demanding that similar inputs be mapped to similar outputs. 

Using normalization (e.g.) of pixel values, we may assume that our samples ${x_{i}}$ are sampled from a convex domain $\Omega_0$, such as $[0,1]^{N^2}$. Our dataset is then of the form
\[
\left\{ {x_{i} ,f\left( {x_{i} } \right)} \right\}_{i\in I} \in \left( {\Omega_0,{\mathbb{R}}^{L}} \right).
\]
We approximate and quantify sparsity of high-dimensional non-smooth functions, using the Wavelet Decomposition of a Random Forest.

\begin{figure}[t]
\begin{center}
   \includegraphics[width=1.\linewidth]{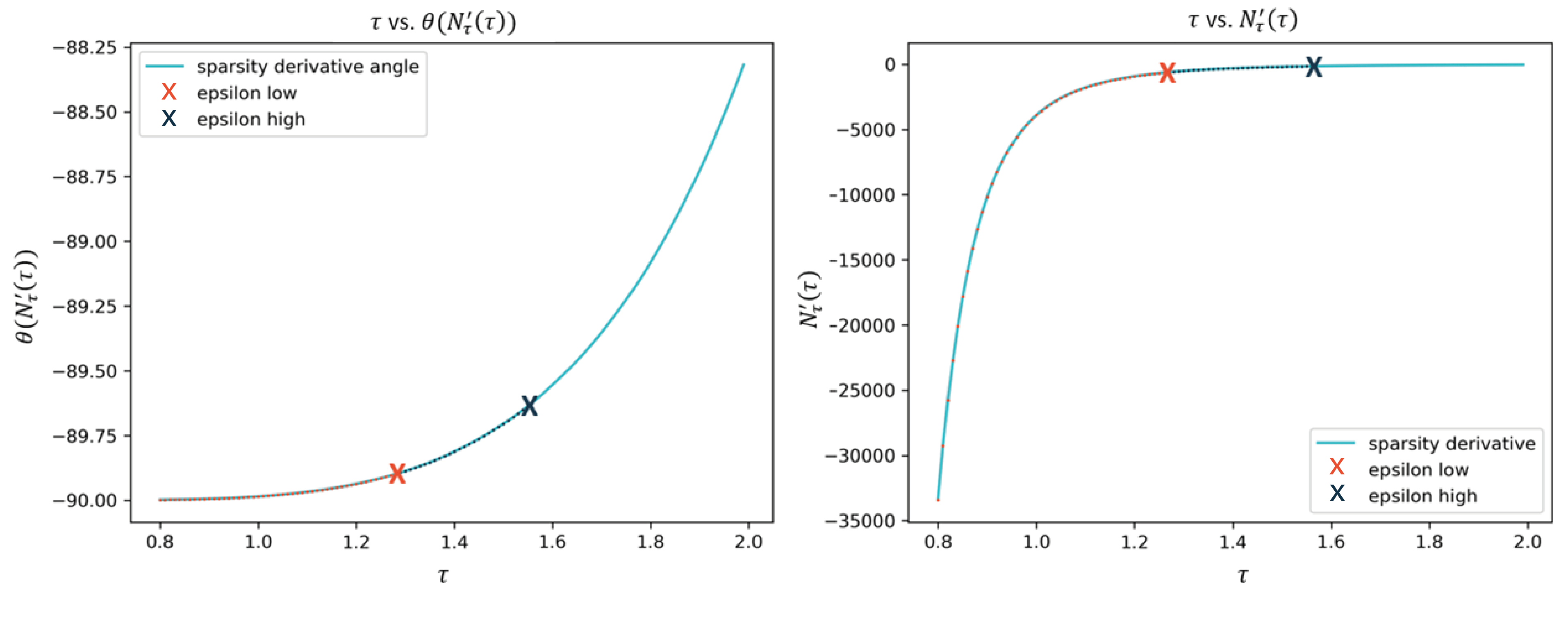}
\end{center}
   \caption{Numerical algorithm for estimating the transition index $\tau^*$. The two meta-parameters, $\varepsilon_{\textrm{low}}$ in red and, $\varepsilon_{\textrm{high}}$ in blue.}
\label{numerics_demonstration}
\label{fig:long}
\label{fig:onecol}
\end{figure}

\subsubsection{Wavelet Decomposition of Random Forest}

We begin with an overview of single trees. Decision Trees aim to find a sparse and efficient representation for the true, underlying function. At each stage, the algorithm seeks the optimal dividing hyperplane, w.r.t a Cost Function. 
The process is continued until a certain stopping criterion is fulfilled. The resulting domains are labeled as \textbf{leaves}. The general setting is as follows:

We begin with the convex domain $\Omega_0$ as the root of the decision tree. At a given stage, for node $\Omega \subset {\mathbb{R}}^{n}$, a cost minimizing partition by a hyperplane is found to split $\Omega$ into  ${\Omega }',{\Omega }''$, $\Omega'\cup {\Omega }''=\Omega$. In the Variance Minimization setting, the cost that is minimized is defined as

\begin{equation}
\begin{aligned}{
\sum\limits_{x_{i} \in {\Omega }'} \left\| {f({x_{i})} -\vec{{E}}_{\Omega'
}} \right\|_{l_{2}({\mathbb{R}}^{L})}^{2} + \sum\limits_{x_{i} \in {\Omega }''} \left\| {f({x_{i})} -\vec{{E}}_{\Omega''
}} \right\|_{l_{2}({\mathbb{R}}^{L})}^{2}}
\label{dis-eq1}
\end{aligned}
\end{equation}

Where: 
\begin{eqnarray*} \label{C_omega}
\vec{{E}}_{\hat{\Omega}} :=\frac{1}{\# \left\{ {x_{i} 
\in \hat{\Omega }} \right\}}\sum\limits_{x_{i} \in \hat{\Omega }} {f\left( {x_{i} 
} \right)}  \qquad \vec{{E}}_{\hat{\Omega}}\in{\mathbb{R}^{L}}
\end{eqnarray*}

The Random Forest \cite{Br1} algorithm is a significant generalization of the single Decision Tree which is a locally `greedy' algorithm. Several decision trees are constructed over random subsets of the training data, and inference is applied through a voting mechanism over the trees. For any point $x\in \Omega_{0} $, the approximation associated with the $j^{th}$ tree, $\tilde{{f}}_{j} \left( {x} \right)$, is computed by finding the leaf $\Omega \in {\mathcal{T}}_{j} $ in which $x$ is 
contained and assigning $\tilde{{f}}_{j} (x):=\vec{E}_\Omega$, where $\vec{E}_{\Omega } $ is the 
corresponding mean value of leaf $\Omega $ computed during training.
The approximate value of a point $x\in \Omega_{0}$ is then given by
\[
\tilde{{f}}\left( x \right)=\frac{1}{J}\sum\limits_{j=1}^J{\tilde{{f}}_{j} 
\left( x \right)}.
\]

We are now ready to define the \textbf{Geometric Wavelet Decomposition of a Random Forest}. Let $\mathcal{T}$ be a decision tree for function $f$. First we denote the `father' wavelet as the constant function $\psi_{\Omega_0}:= \vec{E}_{\Omega_0}$. In going further, let ${\Omega }'$ to be a child of $\Omega $ in tree ${\mathcal T}$, i.e. ${\Omega }'\subset \Omega $ and ${\Omega }'$ was 
created by a partition of $\Omega $. The wavelet $\psi_{{\Omega }'}:\mathbb{R}^{n}\to \mathbb{R}^{L}$ is defined as 
\[
\psi_{{\Omega }'}(x) :={\rm {\bf 1}}_{{\Omega }'}(x) \left( {\vec{{E}}_{{\Omega 
}'} -\vec{{E}}_{\Omega } } \right),
\]
where ${\bf 1}_{\Omega'}$ is the indicator function of $\Omega'$. The wavelet norm is given by 
\begin{equation}
\label{eq5}
\left\| {\psi_{{\Omega }'} } \right\|_{L_2} =\left\| {\vec{{E}}_{{\Omega }'} -\vec{{E}}_{\Omega 
}} \right\|_{l_{2}({\mathbb{R}}^{L})} \left| \Omega' \right|^{1/2}
\end{equation}
Under certain mild conditions on a decision tree $\mathcal{T}$, the following holds \cite{ED}:
\begin{equation}
\label{eq3a}
f=\sum\limits_{\Omega \in {\mathcal T}} {\psi_{\Omega } } 
\end{equation}
We can then define the wavelet decomposition of a RF as: 
\begin{equation}
\label{eq6}
\tilde{{f}}\left( x \right)=\frac{1}{J}\sum\limits_{j=1}^J {\sum\limits_{\Omega \in 
{\mathcal T}_{j} } {\psi_{\Omega } \left( x \right)} } 
\end{equation}

With the Geometric Wavelet Decomposition of a Random Forest at hand, we can define the notion of sparsity. 

\subsubsection{$\tau$-Sparsity}
We define the $\tau$-Sparsity of tree ${\mathcal{T}}$ and parameter $0<\tau<2$, 
\begin{equation}
\label{eq12}
N_{\tau} \left( {f,{\mathcal T}} \right)=\left( {\sum\limits_{\Omega \ne \Omega_0, \Omega \in {\mathcal 
T}} {\left\| {\psi_{\Omega } } \right\|_{2}^{\tau } } } \right)^{1/\tau }:=\|\{  \|\psi_\Omega\|_2   \}_{\Omega\in\mathcal{T}} \|_{l_\tau}
\end{equation}
It is easy to see that:
\begin{equation}
\label{limit_sparsity}
\lim _{\tau \rightarrow 0} N_{\tau}(f, \mathcal{T})^{\tau}=\left\{\# \Omega \in \mathcal{T}: \|\psi_{\Omega}\|_2 \neq 0\right\}:=\|\{  \|\psi_\Omega\|_2   \}_{\Omega\in\mathcal{T}} \|_{0}
\end{equation}
This coincides with sparsity described in \cite{Elad}. Let us further denote the $\tau $-sparsity of a forest ${\mathcal{F}}$, by

\begin{equation}
\label{eq1}
 N_{\tau} \left( {f,{\mathcal F}} \right)  :=\frac{1}{J}\left( {\sum\limits_{j=1}^J {N_{\tau } \left( {f,{\mathcal T}_{j} } 
\right)^{\tau }} } \right)^{1/\tau } 
= \frac{1}{J}\left( 
{\sum\limits_{j=1}^J {\sum\limits_{\Omega \ne \Omega_0, \Omega \in {\mathcal T}_{j} } {\left\| {\psi 
_{\Omega } } \right\|_{p}^{\tau } } } } \right)^{1/\tau }. 
\end{equation}
The $\|.\|_\tau$ norm is monotonically non-increasing in $\tau$.

\subsection{$\tau$-Sparsity Motivation}

\subsubsection{Smooth-curve Separator in $\mathbb{R}^2$}
We begin with a lemma that gives a bound on the $\tau$-sparsity of a smooth curve separator in the binary classification setting, with features in $\mathbb{R}^2$, and a dyadic non-adaptive tree.
\begin{lemma} 
\label{tau-example-2d}
Let $f(x)={\bf 1}_{{\tilde{\Omega}}}(x)$, where $\tilde{\Omega}\subset{{[0,1]}^2}$ is a compact domain with a smooth boundary. Then, for $1<\tau<2$ , $N_{\tau} \left( {f,{{\mathcal T}_I}} \right)<\infty$, where ${{\mathcal T}_I}$ the tree with isotropic dyadic partitions, creating dyadic cubes of area $2^{-2k}$ at level $2k$.
\end{lemma}
\begin{lemma}
\label{existence-tau-index}
Let $f(x)=\sum\limits_{k=1}^K {{c}_{k}}{\bf 1}_{{B}_{k}}(x)$, where ${B}_{k}\subset {\Omega_0}$ are disjoint cubes with sides parallel to main axes, ${c}_{k}\in\mathbb{R}$. Then, there exists an adaptive tree $\mathcal{T}$, such that for every $0<\tau<2$, $N_{\tau} \left( {f,{\mathcal T}} \right)<\infty$.
\end{lemma}

Based on lemma \ref{existence-tau-index} we can define the $\tau^*$, as the transition index
\begin{definition}[transition $\tau$-index]
We define:
\begin{equation} \label{tau-star-def}
\tau^*:=\inf_{0<\tau<2} \left\{\tau | N_{\tau } ( {f,{\mathcal F}} )<\infty \right\},
\end{equation}
where  $N_{\tau } \left( {f,{\mathcal F}} \right)$ is given in \eqref{eq1}.We notice that due to the monotonicity of the $N_{\tau} \left( {f,{\mathcal F}} \right)$, the transition index is the smallest $\tau$ such that the $\tau$-sparsity is finite.
\end{definition}

It was shown in \cite{ED}, \cite{FSA}, that under certain mild conditions, the Forest Besov-Smoothness of 
the function is equivalent to it's $\tau$-sparsity.

\subsection{Numerical estimation of $\tau^*$}
Since $\tau^*$, defined in \eqref{tau-star-def}, is a complicated transition index, the task of estimating it is highly non trivial. Here we propose a more robust method then the methods proposed in \cite{ED}, \cite{FSA}. First, we use (\ref{eq1}), to create a series of samples $N_{\tau_k } \left( {f,{\mathcal F}} \right)$, for a set of discrete samples $\{ \tau_k \}$, $0<\tau_k<2$. We then approximate at these samples numerical derivatives
\[
N'_\tau(\tau_k):=\frac{\partial N_\tau(f,\mathcal F)}{\partial \tau}(\tau_k).
\]
We use the angles of the derivatives
\[
\theta(\tau_k):=\arctan(N'_\tau(\tau_k)), 
\]
to estimate the transition index $\tau^*$. Now, observe that the transition index is associated with an `infinite' derivative, or equivalently an angle of $-\pi/2$. To this end, we use two meta parameters: $\varepsilon_{\textrm{low}}, \varepsilon_{\textrm{high}}$, and define
\[
S:=\{\tau_k: -\frac{\pi}{2}+\varepsilon_{\textrm{low}}\le \theta(\tau_k) \le -\frac{\pi}{2}+\varepsilon_{\textrm{high}} \}.
\]
We now define the transition index by $\tau^*:=\frac{1}{|S|}\sum_{\tau_k\in S}\tau_k$. A demonstration is shown in Figure \ref{numerics_demonstration}.

%-------------------------------------------------------------------------

\section{Experiments}\label{sec:Experiments}
The experiments throughout the paper use  $\varepsilon_{\textrm{low}}=0.1$, $\varepsilon_{\textrm{high}}=0.4$. We use a three trees, with maximal depth $15$. The sparsity-probe is deployed on the input layer, and all intermediate model layers. We do not test the sparsity at the final model layer. For each dataset, each model is trained with three different initialization seeds, and approximated throughout the layers.
\subsection{Sparsity Probe on the synthetic datasets}

We saw that the Linear Probes cannot quantify the separability of the synthetic datasets presented in Figure \ref{synthetic_dataset_fig}. We show the $\tau^*$ estimate on these datasets. Lemma \ref{tau-example-2d} that provides a bound for sparsity using a non adaptive decision tree, suggests that when using adaptive tree partitions, we should expect a  sparsity $\tau^*\leq1$. The numerical estimates for the $\tau^*$ values are reported in Table \ref{tab:Table1}. As the distance between the classes in the Circles dataset is largest, the $\tau^*$ is indeed the lowest. The Spiral dataset has more distance between the classes than in the Gaussian Quantiles, and so its $\tau^*$ is slightly lower. 

\begin{table}\centering
\caption{$\tau$-Sparsity of synthetic datasets. Although the datasets are completely separable by a smooth curve, due to its non-linearity - the Linear Probes cannot quantify this separability.}
\bigskip
\label{tab:Table1}
\begin{tabular}{@{}ccccc@{}}
\toprule
    Dataset&  $\tau^*$ \\
\midrule
   Spiral&   \textbf{0.98}& \\
   Circles& 	\textbf{0.93}&\\
   GQ&   \textbf{0.99}&\\
\bottomrule
\end{tabular}
\end{table}

\subsubsection{Comparison of the Sparsity Probe transition index to Clustering indices}
Although clustering is closely related to sparsity, there are caveats that yield it inaccurate when trying to evaluate the separation of a latent space:
\begin{enumerate}
  \item [(i)]  Most clustering methods do not deal well with non linearly separated data. KMeans, for example cannot handle well the synthetic datasets of Figure \ref{synthetic_dataset_fig}, and its outcome will be similar to those of the Linear Probes. More advanced Hierarchical Clustering methods can be used to improve this issue in some scenarios.  
  \item [(ii)] When we look at the latent spaces of deep intermediate layers, we expect to observe well-separated clusters. However, in the shallow layers this is simply not true and there are scenarios where the geometry of shallow layers is too vague for most clustering indices. 
  \comment{
  \item [(iii)] We introduce the following dataset to show motivation for the clustering metrics comparison. The dataset consists of 10,000 points split into 20 clusters, each with a random std and of a random binary label. We then run a simple KMeans algorithm with $k=2$, and compare the results across various statistics in table \ref{tab:Table2}. Other than the Fowlkes-Mallows Index and the Silhouette scores - the statistics are close to zero. This is because these clustering statistics penalize for a cluster being split into several different clusters. The Fowlkes-Mallows Index and the Silhouette score account for inter-cluster similarity. The $\tau^*$ on this dataset is $\textbf{0.65}$. This can be compared with the scores of \ref{tab:Table1}. As this dataset is much more separated than any of the three synthetic datasets, the $\tau^*$ is indeed far lower. Since we are using trees of finite depth, the $\tau^*$ approximates the divergence to $0$ on these datasets. }
  \end{enumerate}

\comment{
\begin{table}\centering
\caption{Clustering indices of the example with 20 clusters in dimension 2}\\
\bigskip
\label{tab:Table2}
\begin{tabular}{@{}ccccc@{}}
\toprule
    Metric&  Value \\
\midrule
    Rand Index - adjusted for chance& 0.0001 \\
    Adjusted Mutual Information& 7.21e-05\\
    Homogeneity&   1.28e-15\\
    Completeness&   1.28e-15\\
    Fowlkes-Mallows Index&  0.49\\
    Silhouette Score&   0.44 \\
\bottomrule
\end{tabular}
\end{table}}

\subsection{Sparsity-Probe on Neural Networks}
In order to fully approximate the true functional setting of Deep Neural networks, which are known to be unstable\cite{Robust}, we train each network with 3 different seeds, and approximate $\tau^*$ for each of the intermediate layers. We then set the mean of the 3 sparsity index estimates as the estimated index. In some of the figures we also render the certainty intervals of the graphs of $\tau^*$. 

\begin{definition}[$\alpha$-Score]
A closely related definition to $\tau$-sparsity is the $\alpha$-score. For the critical $0<\tau^*<2$, the $\alpha^*$-index is defined as

\begin{equation}
\label{alpha-score}
\alpha^* := \frac{1}{\tau^*} - \frac{1}{2} > 0
\end{equation}
\end{definition}

In the following sections we report the critical $\alpha$-score found by the algorithm. Since they have an inverse relation, we are looking for the highest $\alpha$-score.

\subsubsection{Analyzing the contribution of adding depth to a network}\label{cifar_vgg}
Suppose we wish to create a model for the CIFAR10 dataset, and decide to use a VGG\cite{VGG} architecture. It is natural to ask - how deep should our model be? Too few layers and the accuracy could be low, too many and the model capacity will be too high, and lead to overfitting. Furthermore, can we actually quantify the contribution of each added layer to the outcome? We train the VGG$\{$ 13, 16, 19$\}$ architecture variants on CIFAR10 for 100 Epochs and estimate the sparsity at the output of every MaxPooling layer and of every Convolutional layer beginning from the 5th layer. For VGG$19$, we also report $\alpha^*$ after 50 Epochs. Results are shown in \ref{vgg_cifar10_comparison_fig}. Observe that the Sparsity Probe reveals certain interesting properties of the different architectures:
\begin{enumerate}
    \item As expected, in general, deeper architectures have the capacity to increase the sparsity and this correlates with the accuracy testing results.
    \item However, we see a certain sparsity saturation phenomena with the VGG16 architecture, where the added layers of VGG19 do not drastically improve sparsity. This correlates with the testing results, where both architectures produce similar accuracy. 
    
    \item When comparing the VGG19 trained on $50$ epochs with the same architecture trained for $100$, it is apparent that the separability improves, most noticeably towards the end of the network.
    
    \item We can also see that the MaxPooling layers usually cause a dip in $\alpha^*$. This can be attributed to the fact that max-pooling is a non-learned, coarse discretization layer. 
\end{enumerate}

\begin{figure}[t]
\begin{center}
   \includegraphics[width=1.\linewidth]{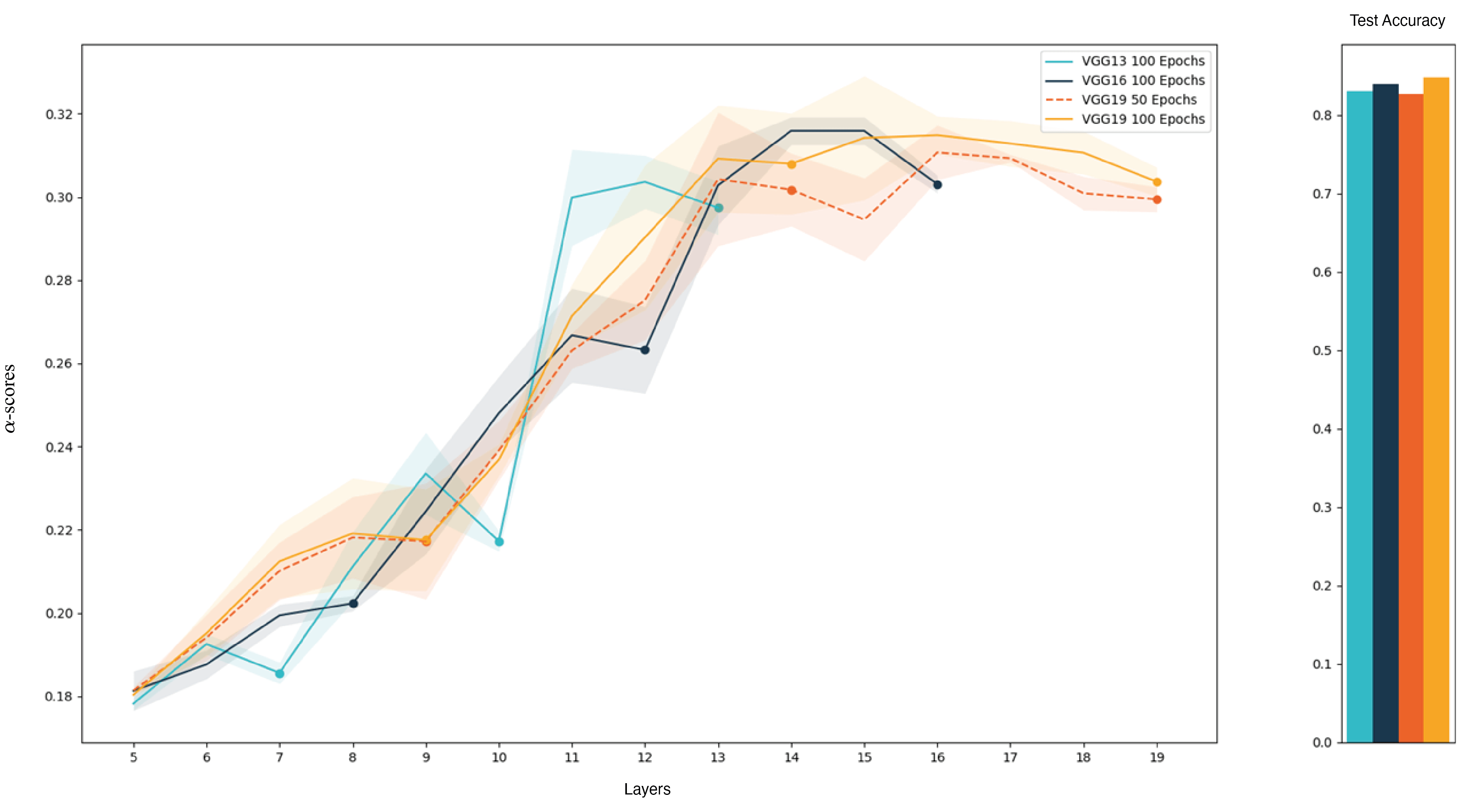}
\end{center}
   \caption{Comparing VGG-$\{$13, 16, 19$\}$ on the CIFAR-10 datasets for $100$ Epochs. VGG19 is also shown after $50$ Epochs. Layers are either MaxPooling(Markers) or Convolutional. We omit the first 5 layers of each network.}
\label{vgg_cifar10_comparison_fig}
\label{fig:long}
\label{fig:onecol}
\end{figure}

\subsubsection{Using the Sparsity Probe to compare different architectures}\label{fashion_mnist}
We wish to asses the probe's ability to analyze problematic architectures that do not perform well on the testing data, e.g, that are not able to generalize \cite{FANTASTIC} (recall that our probe only uses the training data). We report our results on the Fashion-MNIST\cite{fmnist} dataset. As mentioned before, modern architectures in Machine Learning consist of two parts - the feature extractor and the classifier. The main role of the feature extractor is to transform the latent features into more easily separable latent embeddings. Convolutional Layers(ConvLayers) are used to learn increasingly complicated features across layers, by using spatial relations. Suppose one were to create an entirely different architecture that alternates between ConvLayers and linear layers. Every output of a linear layer is transformed into a square as an input into the ConvLayer. Looking at the $\alpha^*$ in Fig. \ref{good_bad_fig}(Left), we see that for the dark-blue line, the scores decrease throughout the layers. The light-blue line reports a model trained on a Resnet18\cite{Resnets} variant, with smaller channel sizes, and inputs of gray-scale images. $\alpha$-scores are measured at the end of every residual connection, and in edge layers.  A general theme which is shown is that the true rise of the $\alpha^*$ happens towards the end of the network. This can be explained by the strength of the gradients that arrive at the earlier layers. 

\begin{figure}[t]
\begin{center}
   \includegraphics[width=1.\linewidth]{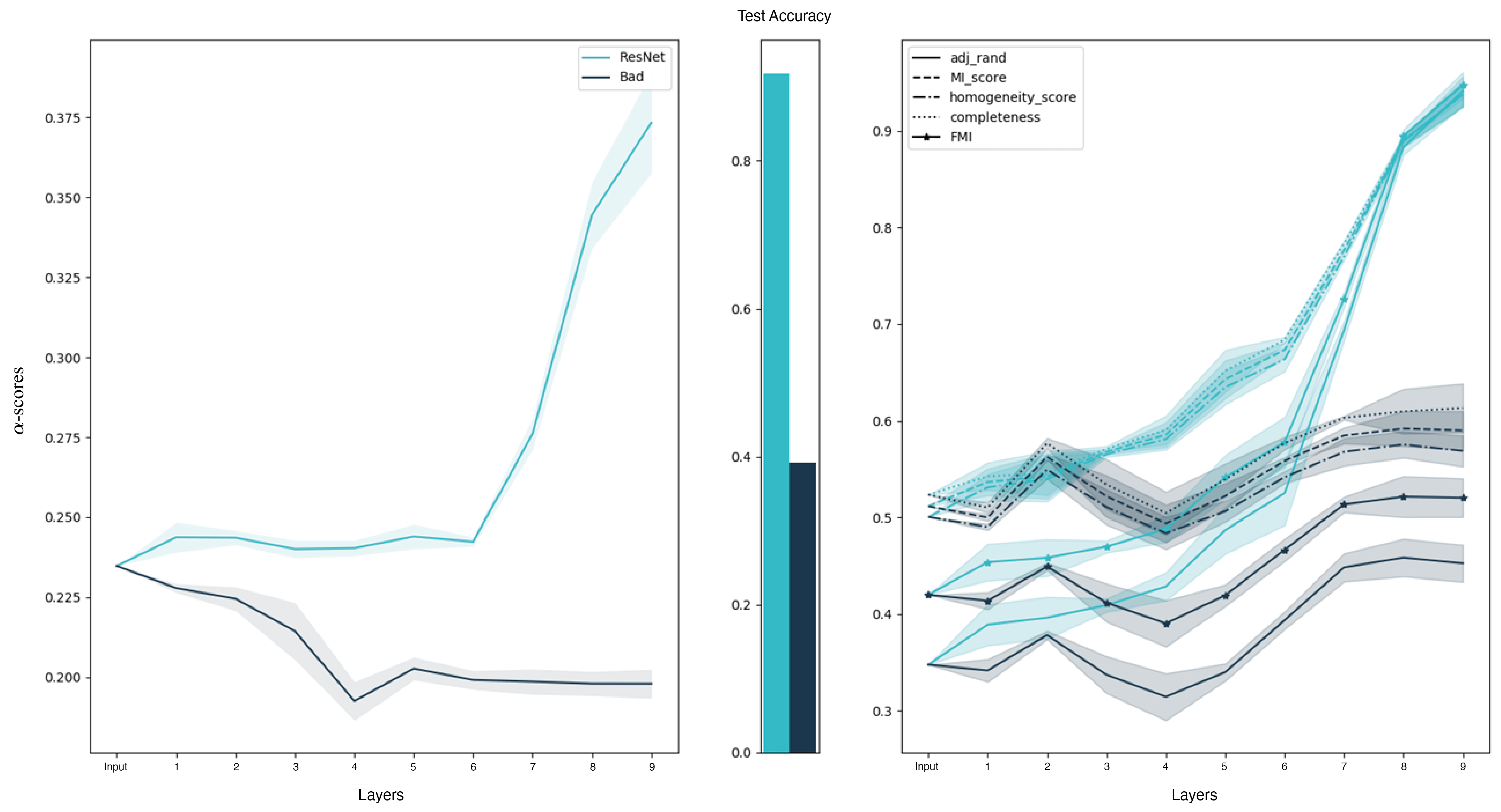}
\end{center}
   \caption{Comparing good and bad architectures $\alpha$-scores on the Fashion-MNIST dataset. The Sparsity-Probe is able to differentiate the quality of the models, and see the inter layer improvements.}
\label{good_bad_fig}
\label{fig:long}
\label{fig:onecol}
\end{figure}

\subsubsection{Fashion-MNIST - Clustering Correlation}
It is natural to ask how $\alpha^*$ behaves compared to clustering statistics. We would expect high correlation between the metrics when the model latent features are well clustered, and low correlation when the features are not well clustered, or clustered into many different clusters per-label.

To test this, we run the KMeans clustering algorithm with $k=10$ on each of the layer features, using the models from the previous section.
Figure \ref{good_bad_fig}(Right) shows the clustering statistics compared in each intermediate layer. In the good models, the correlation between the sparsity and the clustering statistics is positive, demonstrating the data gathers into better label-groups as the layers progress. However, in under performing models, the clustering indices do not manage to asses how bad is the geometry of the represnetations in the intermediate layers. Moreover, in this example, the clustering indices fail to capture the fact that the representations get `worse' throughout the layers. 

We can demonstrate this by using the Pearson Correlation coefficient between each of the clustering statistics and $\alpha^*$ on all model layers. These results are reported in table \ref{tab:dataset_clustering}.

\begin{table}\centering
\caption{Pearson Correlation: $\alpha$-scores vs. Clustering on the Fashion-MNIST dataset}
\bigskip
\label{tab:dataset_clustering}
\begin{tabular}{@{}ccccc@{}}
\toprule
    Metric&Bad&Good \\
\midrule
    Rand Index - adjusted for chance&-0.41&0.96 \\
    Adjusted Mutual Information&-0.41&0.93\\
    Homogeneity&-0.39&0.94\\
    Completeness&-0.43&0.93\\
    Fowlkes-Mallows Index&-0.43&0.96\\
\bottomrule
\end{tabular}
\end{table}

\subsubsection{Case Study - MNIST-1D}\label{mnist_1D}
Recent work \cite{MNIST_1d} propose a 1D parallel to the well-known MNIST. The intention of this dataset is to scale down the dimension of the MNIST, and essentially turn the classes into different signals. The authors also show how, as opposed to MNIST, the signals are intertwined in the input space, and are then much less separable. One could ask - does the model improve throughout the epochs? We use this dataset, with a simple 7-layer Convolutional Neural Network, and monitor the $\alpha^*$ at every 10 epochs. The results are shown in figure \ref{epochs_fig}. It is clear that the model not only improves in the final layer, but is able to create a smoother increase in the intermediate layers, as we progress through the epochs. 

\begin{figure}[t]
\begin{center}
   \includegraphics[width=1.\linewidth]{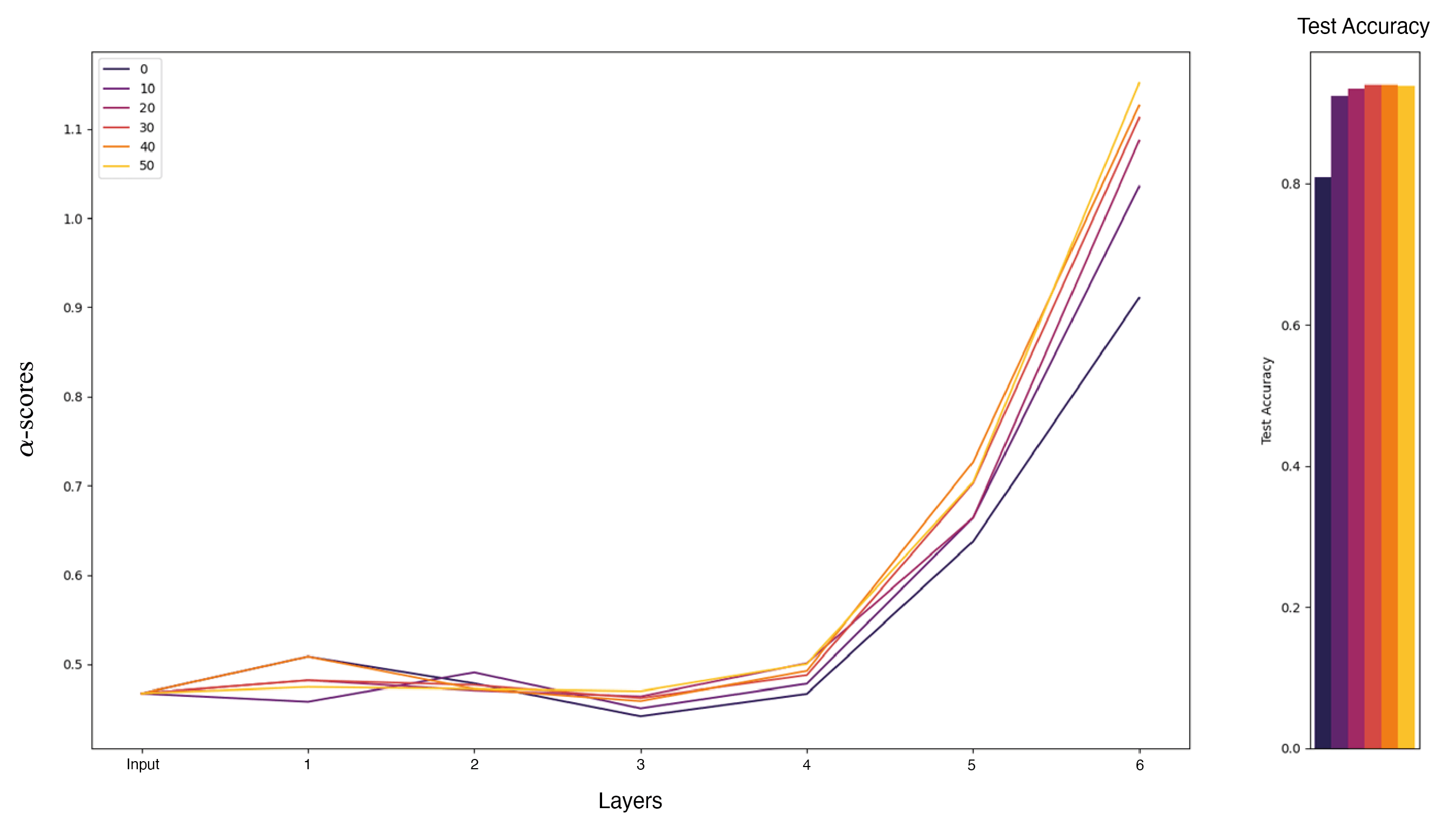}
\end{center}
   \caption{Comparing different epochs during train of a 7-Layer convolutional network on the MNIST-1D dataset. It is clear that the $\alpha$-scores improve through the epochs, and show a smoother rise, with more contribution at every layer. 
    }
\label{epochs_fig}
\label{fig:long}
\label{fig:onecol}
\end{figure}

\subsubsection{Case study- Image classification from the magnitude of Fourier coefficients}\label{phase_mnist}
The problem of Phase Retrieval(PR) is defined as recovering the an image solely using the magnitude of its Fourier transform coefficients. This is a problem that arises in many applications and is obviously an ill-posed inverse problem. Here, we experiment with DL architectures that aim to classify images from the MNIST dataset, again, using only the magnitude of the Fourier coefficients as input. At first, it seems natural to use a standard convolutional network for this classification task. However, as is clear from  Figure \ref{phase_fig}, this approach completely fails. As the plot of the $\alpha^*$ score for this model shows, the network fails to 'unfold' the data and the test accuracy score is very low. The explanation for this phenomena is that architectures based on convolutions assume there are spatial correlations between neighbouring pixels in the input or features in the feature maps. However, this is not true in the Fourier domain. As we see in Figure \ref{phase_fig}, when one applies a fully-connected network architecture, it is able to learn features which are not of spatial nature. We see how the Sparsity Probe is able to capture the fit of the network to the problem. 

\begin{figure}[t]
\begin{center}
   \includegraphics[width=1.\linewidth]{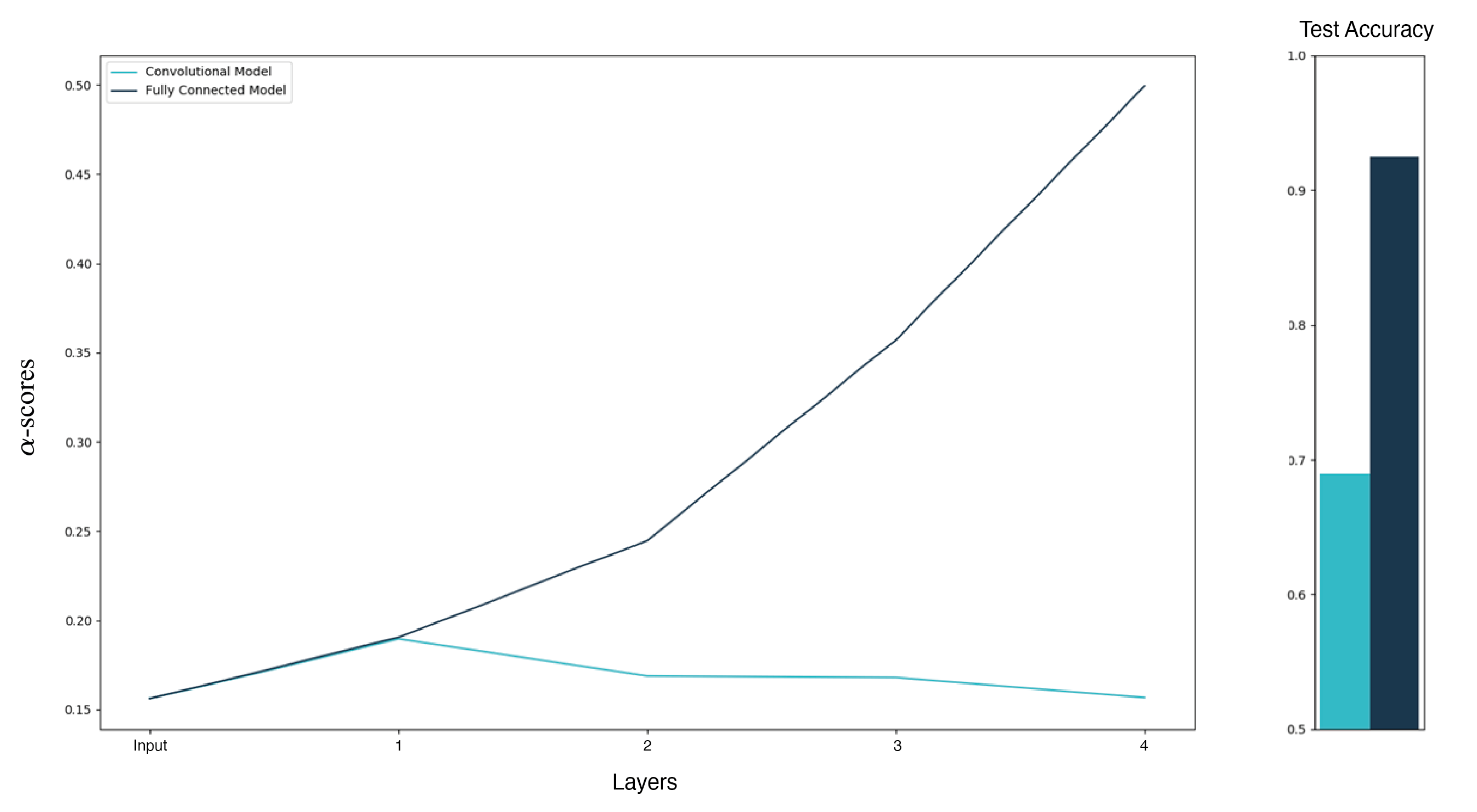}
\end{center}
   \caption{Comparing Fully-Connected and Convolutional architectures on MNIST classification using only Fourier-intensity of the images. It is clear that the $\alpha$-scores of the architectures match the intuition in this task, as the Fourier Domain does not consist of spatial features, and will therefore fail when using convolutional layers.
    }
\label{phase_fig}
\label{fig:long}
\label{fig:onecol}
\end{figure}

\section{Analyzing and debugging architectures}
In this section we leverage the sparsity-probe to analyze model architectures and detect problematic layers. We use a 7-layer Convolutional Network with Batch Norm \cite{2015arXiv150203167I}, RELU activation function, and two linear classification layers. Each model is trained with 3 different initialization seeds, for 100 Epochs on the MNIST-1D dataset. We show direct correlations between the $\alpha^*$ and the test accuracy, even though the $\alpha^*$ are computed solely from the train dataset.

\subsubsection{Picking the Batch Size}
Given a network architecture, meta-parameters still need to be fine-tuned for network training. The batch size is of critical importance as it is constrained by the compute, but also by the affects in the optimization algorithm. If the batch-size is too large, many gradients of different directions can lead to slower convergence. However, if the batch size is too small, the batch can be non-representative of the dataset and lead to a wrong gradient step. In Figure \ref{batch_size_fig} we compare different Batch Sizes. From the analysis, it is apparent that a batch-size of 512 is too large, and results in lower $\alpha^*$, and accordingly test scores. Batch sizes of 64 and 128 behave similarly in terms of $\alpha^*$, with a slight advantage to 64, matching the test-scores. Using the sparsity-probe to investigate, we can understand the trade-off between the computational cost and the added gain.

\begin{figure}[t]
\begin{center}
   \includegraphics[width=1.0\linewidth]{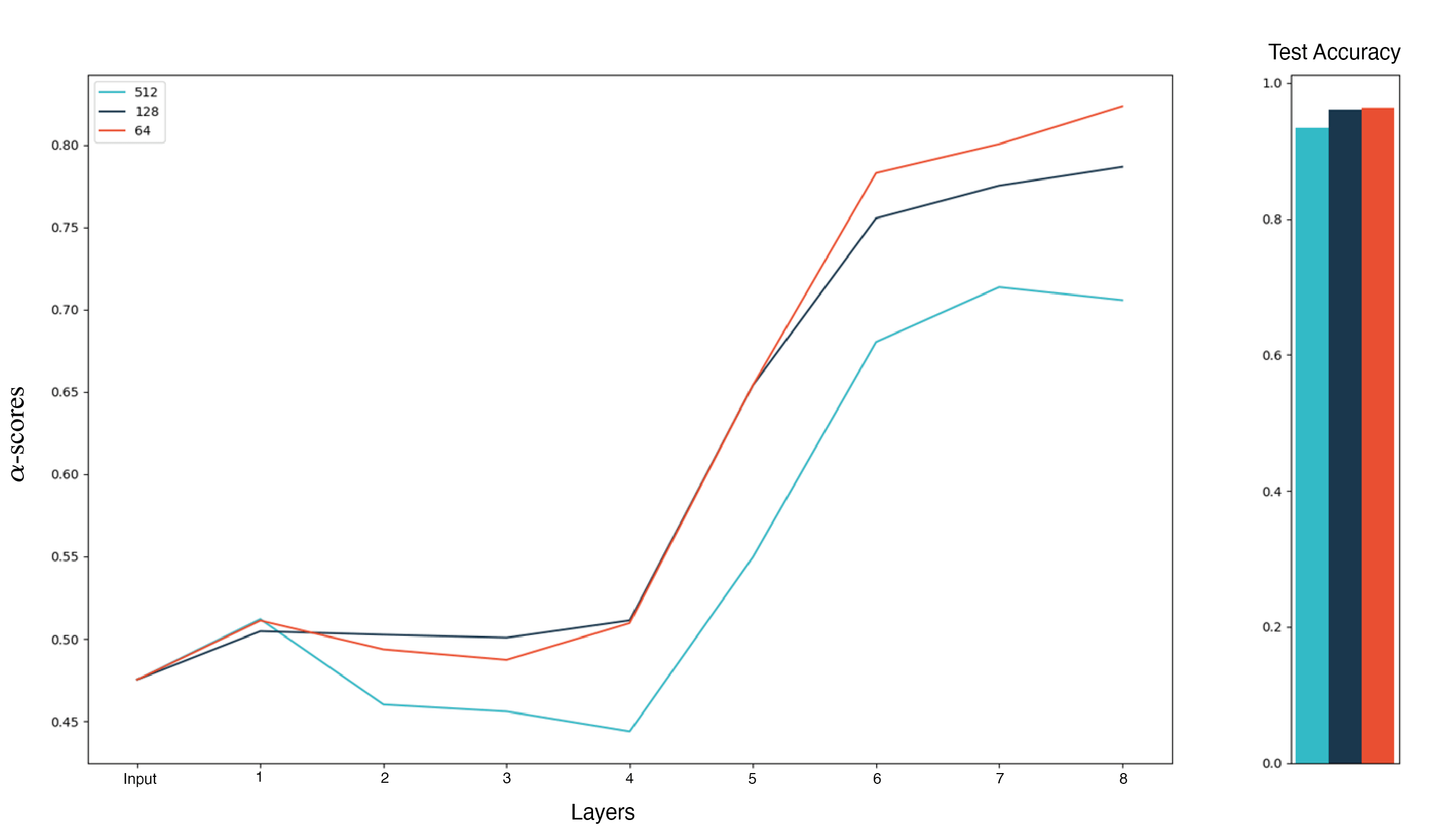}
\end{center}
   \caption{Comparing different Batch-Sizes on the MNIST-1D dataset using the Sparsity-Probe.}

\label{batch_size_fig}
\label{fig:long}
\label{fig:onecol}
\end{figure}

\subsubsection{Picking the Activation Function}
The activation function is one of the most dismissed architecture choices, yet known to be vital. In modern architectures, mostly ReLU activations are used. Suppose we are trying to create a new activation function, by using a simple step function: 
\[
f(x)=\begin{cases}
    1, & \text{if $x>0$}.\\
    0, & \text{otherwise}.
  \end{cases}
\]
This is obviously a problematic activation value, as value magnitudes are not considered. We wish to compare it to other nonlinearities. During first works in Neural Networks, the Sigmoid\cite{10.1007/3-540-59497-3_175} was proposed as a nonlinearity. It was later shown to promote several issues, such as vanishing gradients. Different alternatives to ReLU have been proposed such as Leaky-ReLU, and GELU \cite{2016arXiv160608415H}. We compare these activations on the specific task in figure \ref{activations_fig}. For this particular dataset, we see a dominance of the ReLU to other activations. The ReLU variants - Leaky ReLU and GELU are relatively close in performance. The Sigmoid is indeed far lower in terms of $\alpha^*$, and we see a dip during train. Lastly, the proposed nonlinearity fails magnificently, and the $\alpha^*$ get worse throughout the layers. Remarkably, the Sparsity-Probe perfectly matches the ordering of the test scores!

\begin{figure}[t]
\begin{center}
   \includegraphics[width=1.0\linewidth]{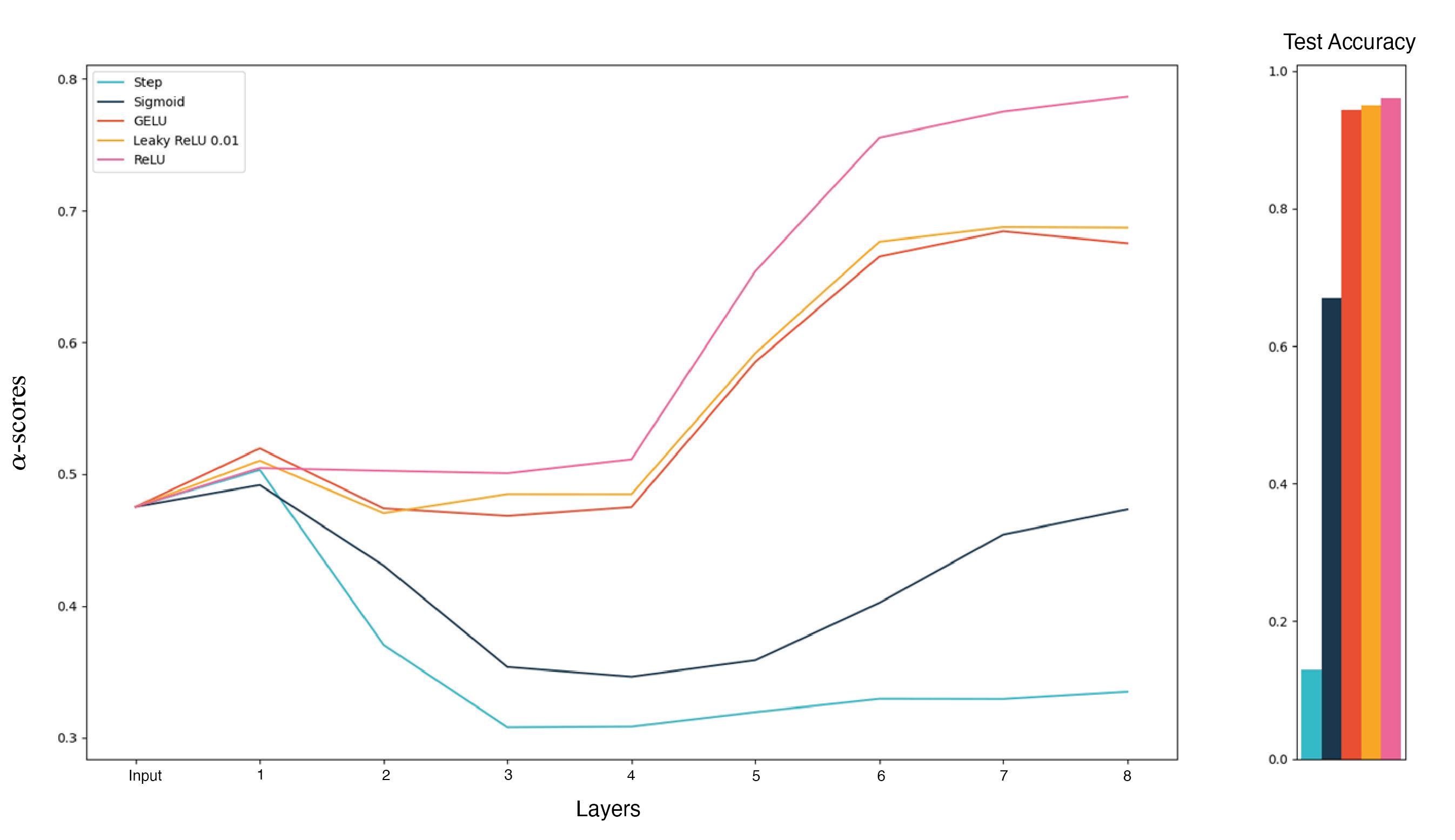}
\end{center}
   \caption{Comparing different Activation Functions on the MNIST-1D dataset using the Sparsity-Probe. We see correlation between the Test Scores and the $\alpha*$ scores in the last level. The proposed 'Step' activations $\alpha^*$ scores deteriorate throughout the layers.}
\label{activations_fig}
\label{fig:long}
\label{fig:onecol}
\end{figure}

\subsubsection{Increasing the Stride}
Suppose we are trying to improve our architecture, by increasing the stride at the $4^\textrm{th}$ layer, from $2$ to $5$, while the kernel size remains the same(3). This would of course result in a loss of information, as the receptive field does not cover the entire input. A comparison of the $\alpha^*$ is shown in Figure \ref{big_stride_fig}. It is clear that during the $4^\textrm{th}$ layer there is a big dip in the $\alpha^*$. This affects the performance of earlier layers, yet we see an increase after the problematic layer. Using our tool, without looking at the test scores, we can pinpoint the exact location where the network fails.

\begin{figure}[t]
\begin{center}
   \includegraphics[width=1.0\linewidth]{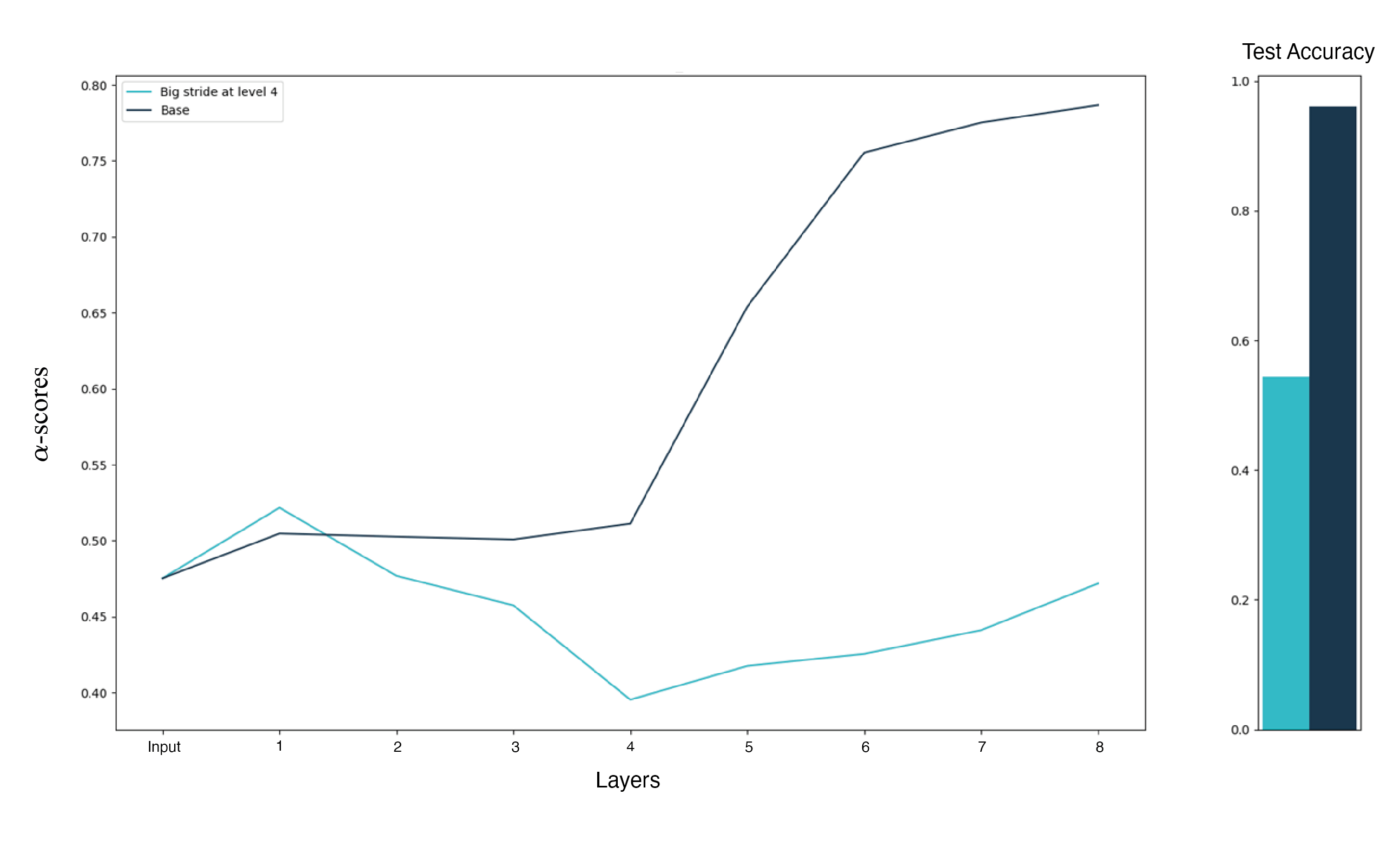}
\end{center}
   \caption{Adding a disproportionately big stride in the 4th layer. We see a sharp decrease in the 4th layer $\alpha^*$-scores, and a slow increase after that. It is interesting to notice that there is a decrease in $\alpha^*$-scores in the first three layers. This is because the gradients are blocked by the 4th layer bottleneck.}
\label{big_stride_fig}
\label{fig:long}
\label{fig:onecol}
\end{figure}

\subsubsection{Batch Normalization}
Batch Normalization has proved extremely helpful for the stability of training, especially in very deep network architectures. Suppose we are trying to test the affects of the Batch Norm(BN), by only applying it at the first $k$ layers. In figure \ref{batch_norm_fig} we report such comparison, with respect to the base architecture, that includes Batch Normalization at all ConvLayers. It is significant to see that the addition of each BN layer increases the $\alpha^*$, and accordingly, the test scores. 

\begin{figure}[t]
\begin{center}
   \includegraphics[width=1.0\linewidth]{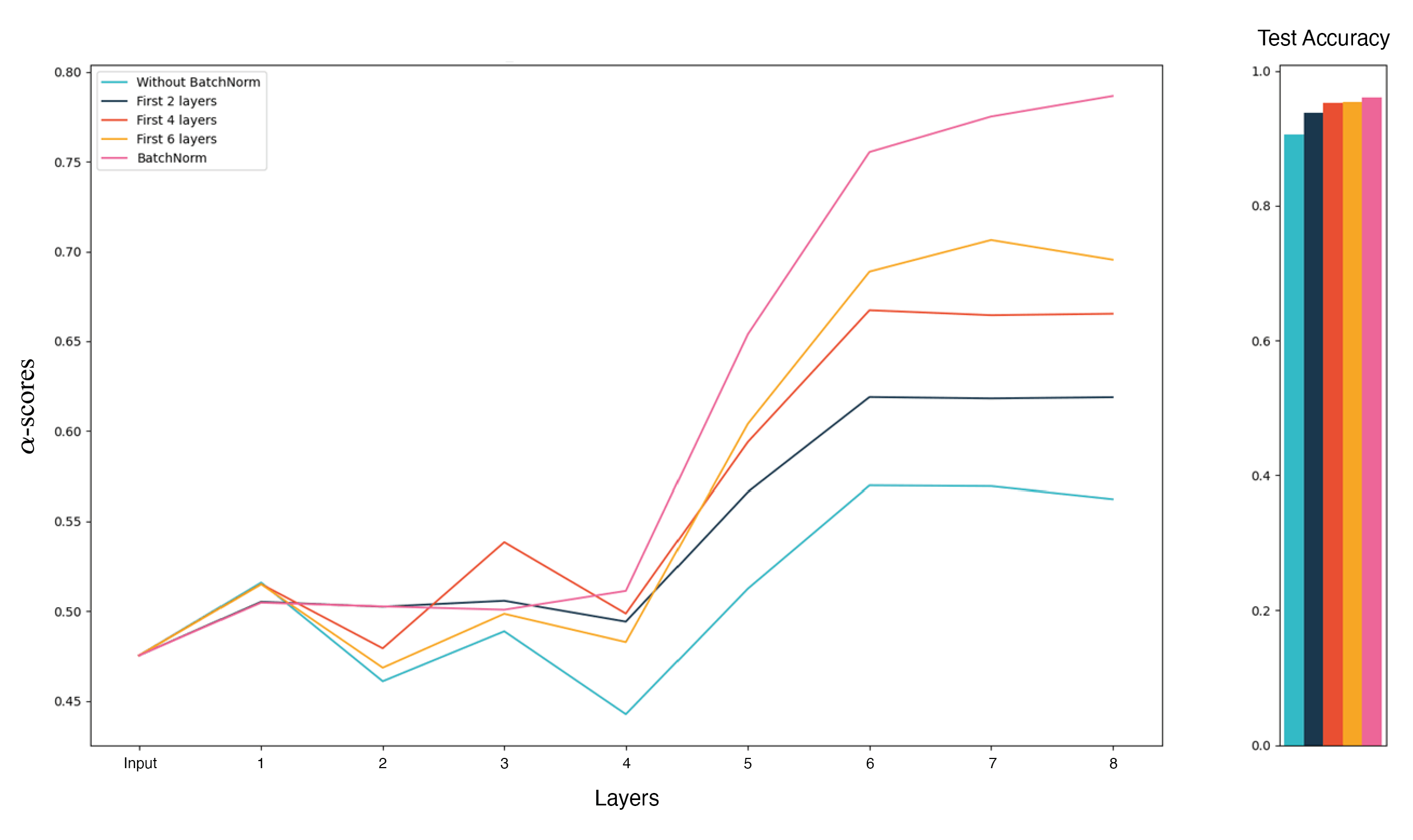}
\end{center}
   \caption{Comparing BatchNorm on a different array of layers using the Sparsity-Probe. We see the added benefit of adding every BN layer.}
\label{batch_norm_fig}
\label{fig:long}
\label{fig:onecol}
\end{figure}

\comment{
\begin{figure}[ht]
\begin{subfigure}[b]{0.25\textwidth}
  \centering
  % include first image
  \includegraphics[width=.8\linewidth]{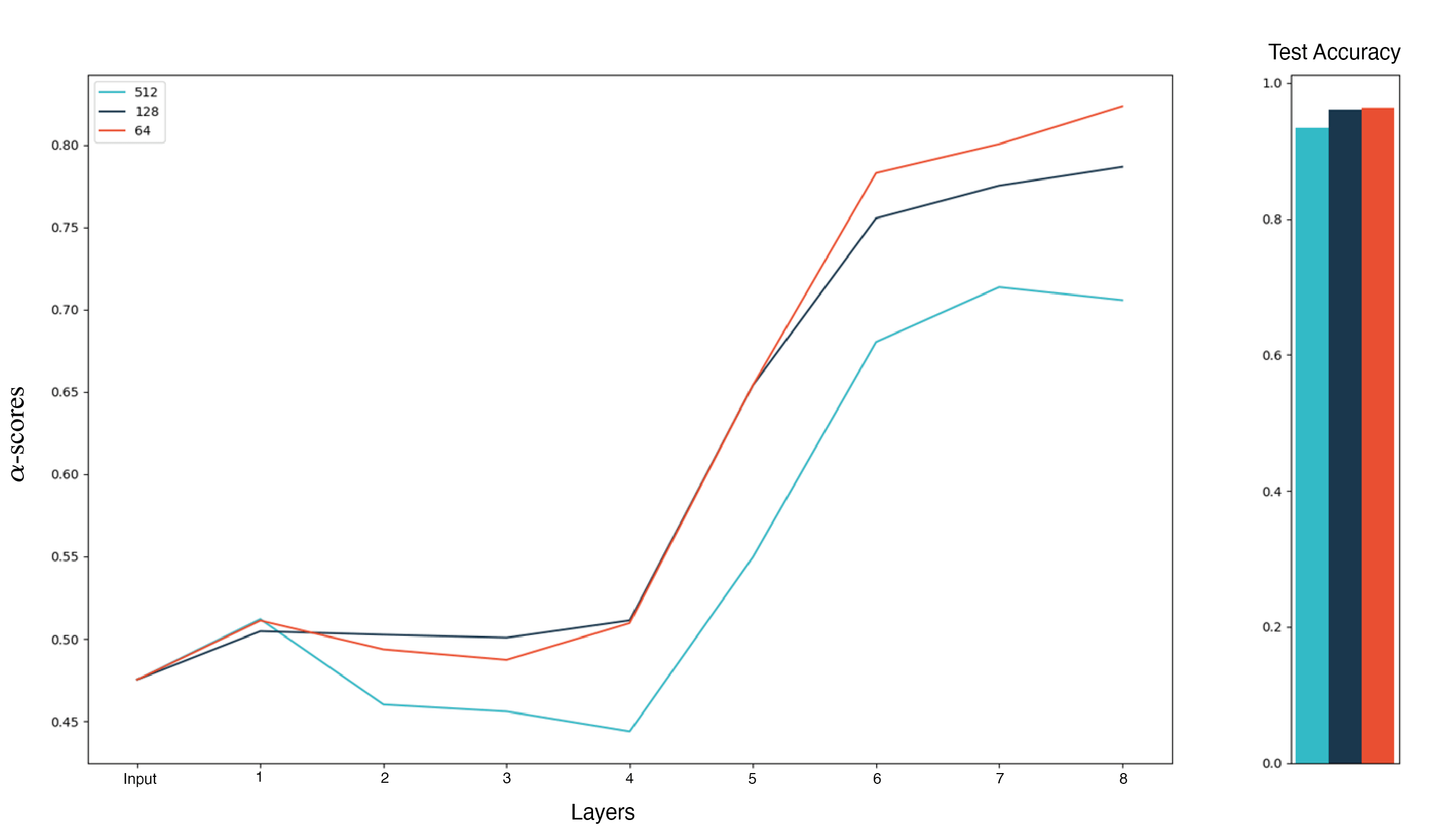}  
  \caption{Put your sub-caption here}
  \label{fig:sub-first}
\end{subfigure}
\begin{subfigure}[b]{0.25\textwidth}
  \centering
  % include second image
  \includegraphics[width=.8\linewidth]{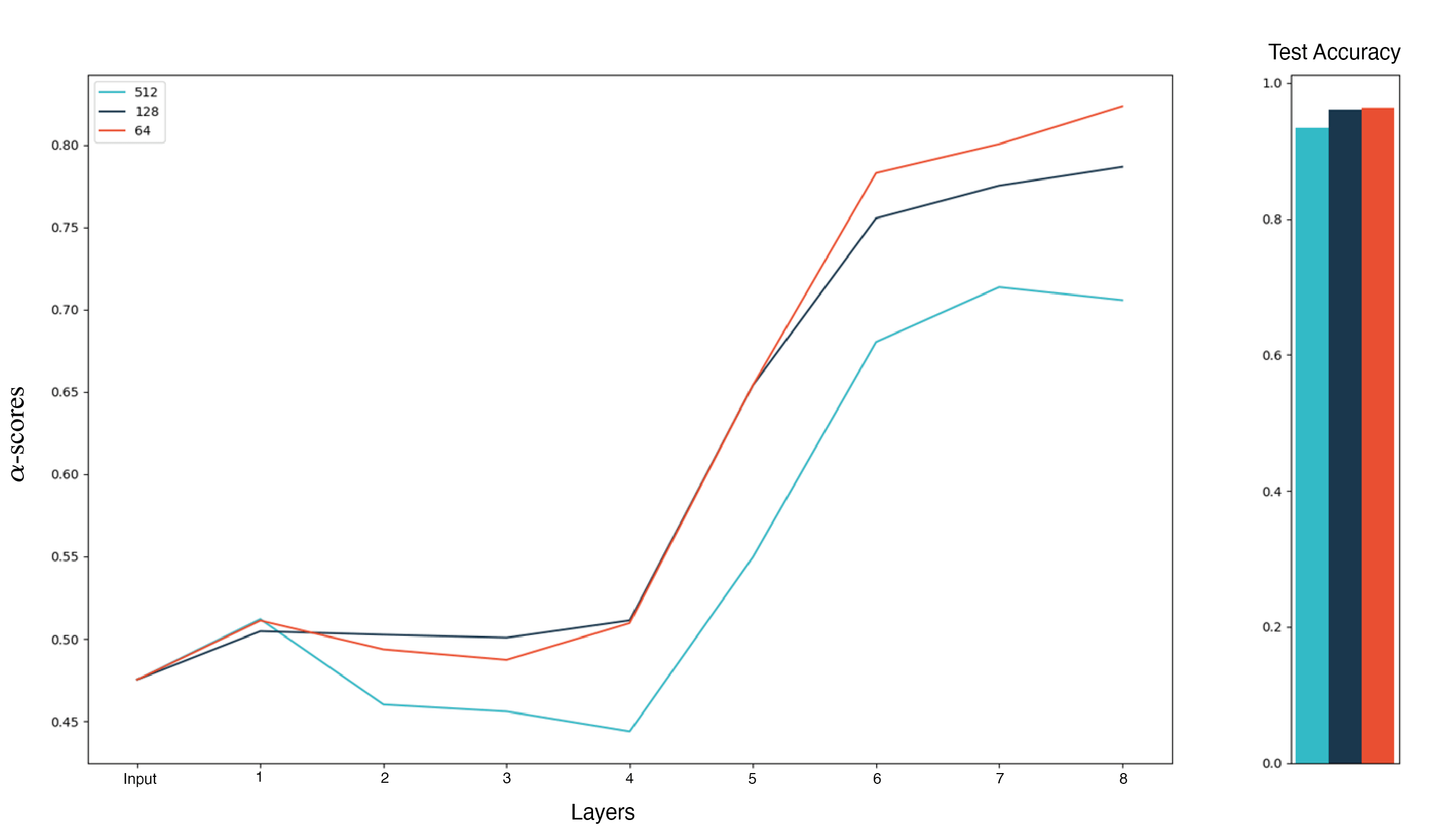}  
  \caption{Put your sub-caption here}
  \label{fig:sub-second}
\end{subfigure}
\caption{Put your caption here}
\label{fig:fig}
\end{figure}
}

\comment{\begin{figure}[t]
\begin{center}
\end{center}
   \caption{$\alpha^*$ scores of a network with \nth{3} activation missing. We see a the $\alpha^*$ remain constant where the activation is missing.
    }
\label{missing_activation}
\label{fig:long}
\label{fig:onecol}
\end{figure}}

\section{Conclusion}
In this paper, we present the \textbf{Sparsity-Probe}, a new method for measuring supervised model quality using sparsity considerations. We give an in-depth explanation of the numerical algorithm and its theoretic background. We give motivation to why the mathematical complexity in this method is necessary for promising results. We show how this method relates to clustering metrics, and show that our method approximates the theoretical bound on a simple 2D synthetic dataset. Our experiments are conducted onto different datasets and have various end goals. We show how the Sparsity-Probe can be used to assess a model quality without an auxiliary test set. This leads to many downstream capabilities such as finding flaws in the architecture and selecting a robust model. 
\comment{
\subsection{Discussion}
Basic results in function theory such as \cite{Tractability} relate the number of samples needed to properly approximate a function in dimension $n_k$. Since the functions we are approximating are possibly very non-smooth and perhaps not even differentiable, the amount of samples needed is exponential. In the Deep Learning domain, $n_k$ is in fact very large, which makes approximating these functions difficult. Given these constraints, further work is needed to improve the Sparsity-Probe to work early layers of large models. Another difficulty is computing the Random Forest of many high-dimensional feature vectors. Future work includes enhancing the Sparsity-Probe to very large models. The Sparsity-Probe gives unique potential to giving mathematical cache to certain network attributes, such as activation function or normalizing units. It can also find flaws in architectures on certain datasets, and give headway on ways to fix them. }

\appendix
\section{Proofs}
\begin{proof} [Proof of Lemma \ref{tau-example-2d}]
Let $\mathcal{T}_I$ be the non-adaptive tree with partitions at dyadic values along the main axes. $\mathcal{T}_I$ partitions ${[0,1]}^2$ into $2^{k}$ rectangles of area $2^{-k}$ on level $k$. Since we are in the binary setting, the output range is the interval $[0,1]$ and for $\Omega\in\mathcal{T}_{I}$, $E_{\Omega}$ is a scalar.

Let $l(\Omega)$ be the level in which domain $\Omega$ was created in $\mathcal{T}_I$. The $\tau$-Sparsity of $\mathcal{T}_{I}$ is given by:
\begin{align*}
N_{\tau} \left( {f,{\mathcal T_{I}}} \right) & := \left( {\sum\limits_{\Omega' \in {\mathcal{T}_{I}}, {\Omega' \ne {[0,1]}^2}} {\left\| {\psi_{\Omega' } } \right\|_{2}^{\tau } } } \right)^{1/\tau } \\ & = \left( {\sum\limits_{\Omega' \in {\mathcal T_{I}}, {\Omega' \ne {[0,1]}^2} } {\left| {{{E}}_{{\Omega}'} -{{E}}_{\Omega 
}} \right|^{\tau } } |\Omega'|^{\frac{\tau}{2}}} \right)^{1/\tau }
\\ & = \left( {\sum\limits_{\Omega' \in {\mathcal T_{I}}, l(\Omega')=k, k>0 } {\left| {{{E}}_{{\Omega}'} -{{E}}_{\Omega 
}} \right|^{\tau } } |\Omega'|^{\frac{\tau}{2}}} \right)^{1/\tau } 
\end{align*}

Let $\Omega'\in{\mathcal{T}_{I}}$, and $\Omega\in{\mathcal{T}_{I}}$ be its parent in $\mathcal{T}_{I}$. If $\Omega'\cap\partial\tilde{\Omega} = \emptyset$, and $\Omega\cap\partial\tilde{\Omega} = \emptyset$ then ${{{E}}_{{\Omega}'} -{{E}}_{\Omega }} = 0$. Otherwise, if $l(\Omega') = k$, 
\begin{align*}
{| {{{E}}_{{\Omega}'} -{{E}}_{\Omega }} |}^{\tau} |\Omega'|^{\frac{\tau}{2}} \le 2^{\frac{-\tau{k}}{2}}
\end{align*}

Therefore, 

\begin{align*}
{N_{\tau} \left( {f,{\mathcal T_{I}}} \right)} ^ {\tau} \le {\sum\limits_{\Omega' \in {\mathcal T_{I}}, k>0 } {2^\frac{-\tau{k}}{2} \# A_{k} }}  
\end{align*}

Where Pa$(\Omega')$ is the parent of $\Omega'$ in ${\mathcal{T}_{I}}$, and
\[
A_{k} := \{\Omega': \Omega'\in{\mathcal{T}_{I}}, l(\Omega')=k, \textrm{Pa}(\Omega')=\Omega, (\Omega'\cap\partial\tilde{\Omega} \ne \emptyset \vee \Omega\cap\partial\tilde{\Omega} \ne \emptyset) \}.
\]

We claim that
\begin{align} \label{amount_of_domains_estimate}
\# A_{k} & = \le C(\tilde{\Omega})2^{\floor*{\frac{k+1}{2}}}.
\end{align}

Where $C(\tilde{\Omega})$ is a constant dependant of the domain $\tilde{\Omega}$. This implies that
\begin{align*}
{N_{\tau} \left( {f,{\mathcal T_{I} }} \right)} ^ {\tau} & \le C(\tilde{\Omega})\left( {\sum\limits_{k=1}^{\infty} {2^{\frac{-\tau k}{2}  +\floor*{\frac{k+1}{2}}}}} \right) \\ & = C(\tilde{\Omega})\left( {\sum\limits_{j=1}^{\infty} {2^{\frac{-\tau (2j)}{2}  +\floor*{\frac{2j+1}{2}}}}} + {\sum\limits_{j=1}^{\infty} {2^{\frac{-\tau (2j+1)}{2}  +\floor*{\frac{2j+2}{2}}}}} \right) \\ & = C(\tilde{\Omega})\left( (1 +2^{1-\frac{\tau}{2}}){\sum\limits_{j=1}^{\infty} {2^{-j(\tau-1) }}} \right)
\end{align*}

and so,
\[
N_{\tau} \left( {f,{\mathcal T_{I}}} \right)<\infty \Leftrightarrow \tau > 1
\]

Let us return to the estimate \ref{amount_of_domains_estimate}. We define: 
\[
B_{k} := \{\Omega': \Omega'\in{\mathcal{T}_{I}}, l(\Omega')=k, \Omega'\cap\partial\tilde{\Omega} \ne \emptyset \}.
\]

We notice that if: 
\begin{align}
\label{B_k_estimate}
\#B_k \le  C(\tilde{\Omega}) 2^{\floor*{\frac{k+1}{2}}},    
\end{align}

then the following relation holds: 
\begin{align*}
\#A_k & \le \# B_{k} + \# B_{k+1} \\ & \le C_{1}(\tilde{\Omega}) 2^{\floor*{\frac{k+1}{2}}} + C_{2}(\tilde{\Omega}) 2^{\floor*{\frac{k+2}{2}}} \\ &  \le C(\tilde{\Omega}) 2^{\floor*{\frac{k+1}{2}}}.
\end{align*}

Let us now show \ref{B_k_estimate}.  First, we notice that it is enough to show that for every even layer $2k$, $\# B_{2k} < C(\tilde{\Omega}) 2^{k} = C(\tilde{\Omega}) 2^{\floor*{\frac{2k+1}{2}}}$. Once this is shown, for every odd layer, $2k+1$, the amount of domain intersections with the $\partial\tilde{\Omega}$ is at most the number of intersections in the next layer: 
\[
\# B_{2k+1} \le \#B_{2k+2} \le C(\tilde{\Omega}) 2^{k+1} = C(\tilde{\Omega}) 2^{\floor*{{\frac{2k+1}{2}}}}
\]
Let us look at the boundary $\tilde{\Omega}$ at an even layer $2k$. There are a finite number of points where the gradient of the boundary is aligned with one of the main axes. Between these points the boundary segments are monotone in $x_1$ and $x_2$ and therefore the amount of cubes it intersects is at most $2\times2^k$. This is because on axis $x_1$, the boundary is monotone, and so it can intersect at most $2^k$ dyadic cubes. The same goes for axis $x_2$. We can then bound the total number of intersections at level $2k$ by $C(\tilde{\Omega})2^k$, where $C(\tilde{\Omega})$ is determined by the number of points where the boundary gradient is aligned with one of the main axes. This bound is visualized in figure \ref{first_proof_fig}. 
\end{proof}

\begin{figure}[t]
\begin{center}
   \includegraphics[width=0.8\linewidth]{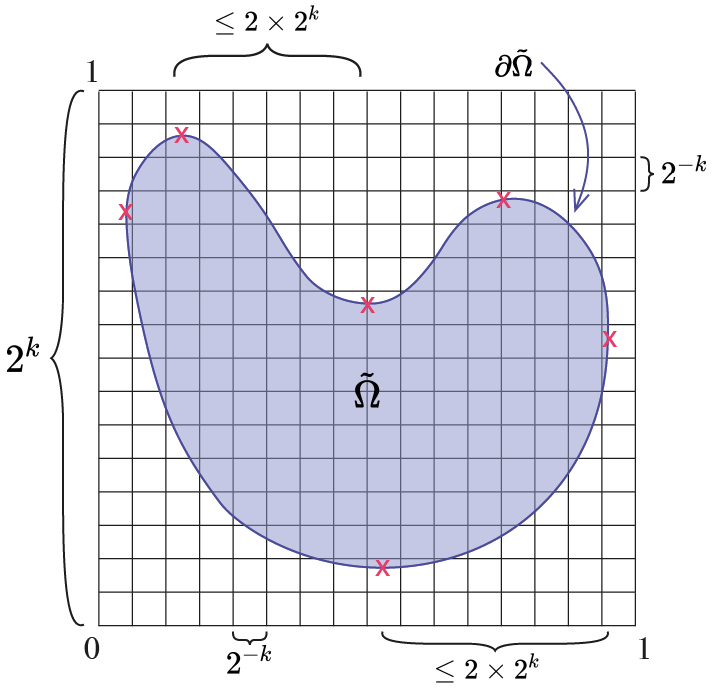}
\end{center}
   \caption{Visualization of bound on the number of cubes that intersect with the boundary $\partial\tilde{\Omega}$, ($\# B_{2k}$) at level $2k$.
    }
\label{first_proof_fig}
\label{fig:long}
\label{fig:onecol}
\end{figure}

\begin{proof} [Proof of Lemma \ref{existence-tau-index}]
Let $f(x)=\sum\limits_{k=1}^K {{c}_{k}}{\bf 1}_{{B}_{k}}(x)$, where ${B}_{k}\subset {\Omega_0}$ are disjoint cubes with sides parallel to main axes, ${c}_{k}\in\mathbb{R}$. Since the cubes $\{{B}_{k}\}$ are disjoint, it is possible to find a tree ${\mathcal T}_{A}$ such that at level $N_0$, each cube ${B}_{k}$ is in a separate domain ${\Omega}_{{B}_k}$. The sides of ${B}_{k}$ are parallel to the main axes, and so it is possible to partition ${\Omega}_{{B}_k}$ 4 times, such that the resulting partition ${\Omega'}_{{B}_k}$ is exactly ${{B}_k}$ with value $C_{k}$. We notice, that any partition after this level, $l=N_1$, the value of $f$ does not change, and therefore the for wavelets $\psi_{\Omega'}$
\[
{\vec{{E}}_{{\Omega }'} -\vec{{E}}_{\Omega }} = \vec{0}
\]
and, 
\[
    \left\| {\psi_{{\Omega }'} } \right\|_{L_2} =\left\| {\vec{{E}}_{{\Omega }'} -\vec{{E}}_{\Omega }} \right\|_{l_{2}({\mathbb{R}}^{L})} \left| \Omega' \right|^{1/2} = 0.
\]

Therefore, 

\begin{align*}
{N_{\tau} \left( {f,{{\mathcal T}_{A}}} \right)} ^ \tau& :=  {\sum\limits_{\Omega' \in {\mathcal{T}_{A}}, {\Omega' \ne {[0,1]}^n}} {\left\| {\psi_{\Omega' } } \right\|_{2}^{\tau } } } \\ & =  {\sum\limits_{\Omega' \in {\mathcal T_{A}}, {\Omega' \ne {[0,1]}^n}, l(\Omega')\le N_{1} } {\left\| {{{E}}_{{\Omega}'} -{{E}}_{\Omega 
}} \right\|_{l_2}^{\tau } } |\Omega'|^{\frac{\tau}{2}}} \\ & < \infty
\end{align*}

where the last transition is true since we are adding a finite amount of finite values. 
A visualization of a possible ${\mathcal T}_{A}$ is shown in figure \ref{second_proof_fig}.
\end{proof}

\begin{figure}[t]
\begin{center}
   \includegraphics[width=0.8\linewidth]{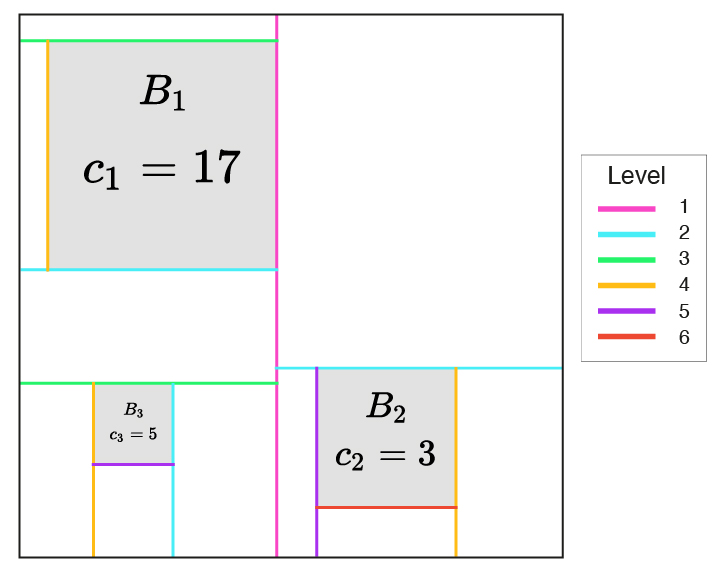}
\end{center}
   \caption{Example adaptive tree ${\mathcal{T}}_{A}$ on  $f(x)=\sum\limits_{k=1}^K {{c}_{k}}{\bf 1}_{{B}_{k}}(x)$. 
    }
\label{second_proof_fig}
\label{fig:long}
\label{fig:onecol}
\end{figure}

\bibliographystyle{siamplain}
\bibliography{SIMODS_main}
\end{document}

%%%%%%%%% TITLE
\twocolumn[
\icmltitle{Sparsity-Probe - Supplementary Material}

\icmlsetsymbol{equal}{*}

\begin{icmlauthorlist}
\icmlauthor{Aeiau Zzzz}{equal,to}
\icmlauthor{Bauiu C.~Yyyy}{equal,to,goo}
\icmlauthor{Cieua Vvvvv}{goo}
\icmlauthor{Iaesut Saoeu}{ed}
\icmlauthor{Fiuea Rrrr}{to}
\icmlauthor{Tateu H.~Yasehe}{ed,to,goo}
\icmlauthor{Aaoeu Iasoh}{goo}
\icmlauthor{Buiui Eueu}{ed}
\icmlauthor{Aeuia Zzzz}{ed}
\icmlauthor{Bieea C.~Yyyy}{to,goo}
\icmlauthor{Teoau Xxxx}{ed}
\icmlauthor{Eee Pppp}{ed}
\end{icmlauthorlist}

\icmlaffiliation{to}{Department of Computation, University of Torontoland, Torontoland, Canada}
\icmlaffiliation{goo}{Googol ShallowMind, New London, Michigan, USA}
\icmlaffiliation{ed}{School of Computation, University of Edenborrow, Edenborrow, United Kingdom}

\icmlcorrespondingauthor{Cieua Vvvvv}{c.vvvvv@googol.com}
\icmlcorrespondingauthor{Eee Pppp}{ep@eden.co.uk}

\icmlkeywords{Machine Learning, ICML}

\vskip 0.3in
]

\section{Proofs}
\begin{proof}[Proof of Lemma ]%\ref{tau-example-2d}]
Let $\cal{T}_I$ be the non-adaptive tree with partitions at dyadic values along the main axes. $\cal{T}_I$ partitions $\tilde{\Omega}$ into $2^{2j}$ dyadic cubes of side length $2^{-2j+1}$ and area $2^{-2j}$ on level $2j$. Since we are in the binary setting, the output range is the interval $[0,1]$ and for $\Omega\in\cal{T}$, $\vec{E_{\Omega}}$ is a scalar.

\begin{aligned}

N_{\tau, p} \left( {f,{\cal T}} \right) &:= \left( {\sum\limits_{\Omega' \ne \tilde{\Omega}, \Omega' \in {\cal 
T}, {\partial{\tilde{\Omega}}}\cap{\Omega'} \ne {\emptyset}} {\left\| {\psi_{\Omega' } } \right\|_{p}^{\tau } } } \right)^{1/\tau } & =\left( {\sum\limits_{\Omega' \ne \tilde{\Omega}, \Omega' \in {\cal 
T}, {\partial{\tilde{\Omega}}}\cap{\Omega'} \ne {\emptyset}} {\left\| {{{E}}_{{\Omega}'} -{{E}}_{\Omega 
}} \right\|_{l_2}^{\tau } } |\Omega'|^{\frac{\tau}{p}}} \right)^{1/\tau } &= \left( {\sum\limits_{l(\Omega')=j, j \ge 0, {\partial{\tilde{\Omega}}}\cap{\Omega'} \ne {\emptyset}} {\left\| {{{E}}_{{\Omega}'} -{{E}}_{\Omega 
}} \right\|_{l_2}^{\tau } } |\Omega'|^{\frac{\tau}{p}}} \right)^{1/\tau }

\end{aligned}

Let $\Omega'\in\cal{T}$ be in level $2j$, and $\Omega$ the parent. Since $E_{\Omega}, E_{\Omega'}\in[0,1]$, $\|{{E}}_{\Omega'} -{E}_{\Omega}\|_{l_2} \le 1$. Further, $|\Omega'|^{\frac{\tau}{p}}=(2^{-2j})^{\frac{\tau}{p}}=2^{\frac{-2j\tau}{p}}.$ Which yields: 

\begin{align*}
% \begin{equation}
     \|{{E}}_{\Omega'} -{E}_{\Omega}\|_{l_2}^{\tau } |\Omega'|^{\frac{\tau}{p}} \le 2^{\frac{-2j\tau}{p}}
% \end{equation}
\end{align*}

Let $\Omega'\in\cal{T}$ be in level $2j+1$, and $\Omega$ the parent. For each such $\Omega'$, $|\Omega'|=\frac{|\Omega|}{2}$. We notice $|\{\Omega':l(\Omega')=2j+1, \Omega'\cap\tilde{\Omega} \ne \emptyset\}| \le 2|\{\Omega':l(\Omega')=2j, \Omega'\cap\tilde{\Omega} \ne \emptyset\}|$. Therefore:
\begin{aligned}

{\sum\limits_{l(\Omega')=2j+1, {\partial{\tilde{\Omega}}}\cap{\Omega'} \ne {\emptyset}} {\left\| {{{E}}_{{\Omega}'} -{{E}}_{\Omega }} \right\|_{l_2}^{\tau } } |\Omega'|^{\frac{\tau}{p}}}&\le{\sum\limits_{l(\Omega')=2j+1} {1^\tau} |\Omega'|^{\frac{\tau}{p}}}|\{\Omega':l(\Omega')=2j+1, \Omega'\cap\tilde{\Omega} \ne \emptyset\}|&\le{\sum\limits_{l(\Omega')=2j} {1^\tau} |\Omega'|^{\frac{\tau}{p}}}|\{\Omega':l(\Omega')=2j, \Omega'\cap\tilde{\Omega} \ne \emptyset\}|

\end{aligned}
We can then look at only the even layers. We also notice that for $\Omega'\in\cal{T}$ in level $2j$, $|\{\Omega':l(\Omega')=2j, \Omega'\cap\tilde{\Omega} \ne \emptyset\}|\le2^j$.

\begin{aligned}

C\left( {\sum\limits_{l(\Omega')=2j, j \ge 0} {\left\| {{{E}}_{{\Omega}'} -{{E}}_{\Omega 
}} \right\|_{l_2}^{\tau } } |\Omega'|^{\frac{\tau}{p}}} \right)^{1/\tau } \le C\left( {\sum\limits_{j \ge 0} {1^\tau} |\Omega'|^{\frac{\tau}{p}}}|\{\Omega':l(\Omega')=2j, \Omega'\cap\tilde{\Omega} \ne \emptyset\}| \right)^{1/\tau } &\le C({\sum\limits_{j \ge 0} {2^{\frac{-2j\tau}{p}+j}}})^{1/\tau }

\end{aligned}
and so, 
\begin{aligned}
N_{\tau, p}<\infty \Leftrightarrow \frac{-2\tau}{p}+1<0 \Leftrightarrow \frac{p}{2}<\tau
\end{aligned}

\end{proof}

\section{Experiment Settings}

%% file: SIMODS_shared.tex
\usepackage{lipsum}
\usepackage{amsfonts}
\usepackage{graphicx}
\usepackage{epstopdf}
\usepackage{algorithmic}
\ifpdf
  \DeclareGraphicsExtensions{.eps,.pdf,.png,.jpg}
\else
  \DeclareGraphicsExtensions{.eps}
\fi

\usepackage{enumitem}
\setlist[enumerate]{leftmargin=.5in}
\setlist[itemize]{leftmargin=.5in}

\newsiamremark{remark}{Remark}
\newsiamremark{hypothesis}{Hypothesis}
\crefname{hypothesis}{Hypothesis}{Hypotheses}
\newsiamthm{claim}{Claim}

\author{Ido Ben-Shaul, Shai Dekel  \\
Department of Applied Mathematics\\
Tel-Aviv University\\
}

\headers{Sparsity-Probe}{I. Ben-Shaul, S. Dekel}

\title{Sparsity-Probe: \\
          Analysis tool for Deep Learning Models \thanks{Received by the editors May 14, 2021.}}

\author{Ido Ben-Shaul\thanks{Department of Applied Mathematics, Tel-Aviv University(\email{idobenshaul@mail.tau.ac.il}, \email{shai@tauex.tau.ac.il}).}
\and Shai Dekel\footnotemark[2]}
\usepackage{amsopn}